\documentclass[10pt,journal,compsoc]{IEEEtran}
\usepackage{ifpdf}
\ifpdf
\else
\fi

\ifCLASSOPTIONcompsoc
  \usepackage[nocompress]{cite}
\else
  \usepackage{cite}
\fi
\ifCLASSINFOpdf
  \usepackage[pdftex]{graphicx}
  \graphicspath{{../pdf/}{../jpeg/}}
  \DeclareGraphicsExtensions{.pdf,.jpeg,.png}
\else
  \usepackage[dvips]{graphicx}
  \graphicspath{{../eps/}}
  \DeclareGraphicsExtensions{.eps}
\fi
\usepackage{amsmath, amssymb, amsthm}
\usepackage{mathtools}
\usepackage{algorithm}
\usepackage{algorithmic}

\usepackage{array}

\ifCLASSOPTIONcompsoc
 \usepackage[caption=false,font=footnotesize,labelfont=sf,textfont=sf]{subfig}
\else
 \usepackage[caption=false,font=footnotesize]{subfig}
\fi

\usepackage{stfloats}
\ifCLASSOPTIONcaptionsoff
 \usepackage[nomarkers]{endfloat}
\let\MYoriglatexcaption\caption
\renewcommand{\caption}[2][\relax]{\MYoriglatexcaption[#2]{#2}}
\fi
\usepackage{url}

\usepackage{bbm}
\usepackage{xcolor}
\usepackage{kotex}

\usepackage{booktabs}
\usepackage{multirow}
\newcolumntype{P}[1]{>{\centering\arraybackslash}p{#1}}

\newtheorem{theorem}{Theorem}

\newtheorem{lemma}{Lemma}
\newtheorem{corollary}{Corollary}
\theoremstyle{definition}

\newtheorem{assumption}{Assumption}
\theoremstyle{remark}

\newcommand\numberthis{\addtocounter{equation}{1}\tag{\theequation}}

\begin{document}
%
\title{Sparse-SignSGD with Majority Vote for Communication-Efficient Distributed Learning}
%
%
%
%

\author{Chanho~Park,~\IEEEmembership{Student Member,~IEEE,}
        and~Namyoon~Lee,~\IEEEmembership{Senior Member,~IEEE}
\IEEEcompsocitemizethanks{\IEEEcompsocthanksitem C. Park is with the Department
of Electrical Engineering, Pohang University of Science and Technology, Pohang, 37673 South Korea.\protect\\
E-mail: chanho26@postech.ac.kr.
\IEEEcompsocthanksitem N. Lee is with the School of Electrical Engineering, Korea University, Seoul, 02841 South Korea.\protect\\
E-mail: namyoon@korea.ac.kr.}
}

%
%

\markboth{}%
{Shell \MakeLowercase{\textit{et al.}}: Bare Demo of IEEEtran.cls for Computer Society Journals}
%



\IEEEtitleabstractindextext{%
\begin{abstract}

The training efficiency of complex deep learning models can be significantly improved through the use of distributed optimization. However, this process is often hindered by a large amount of communication cost between workers and a parameter server during iterations. To address this bottleneck, in this paper, we present a new communication-efficient algorithm that offers the synergistic benefits of both sparsification and sign quantization, called ${\sf S}^3$GD-MV. The workers in ${\sf S}^3$GD-MV select the top-$K$ magnitude components of their local gradient vector and only send the signs of these components to the server. The server then aggregates the signs and returns the results via a majority vote rule. Our analysis shows that, under certain mild conditions, ${\sf S}^3$GD-MV can converge at the same rate as signSGD while significantly reducing communication costs, if the sparsification parameter $K$ is properly chosen based on the number of workers and the size of the deep learning model. Experimental results using both independent and identically distributed (IID) and non-IID datasets demonstrate that the ${\sf S}^3$GD-MV attains higher accuracy than signSGD, significantly reducing communication costs. These findings highlight the potential of ${\sf S}^3$GD-MV as a promising solution for communication-efficient distributed optimization in deep learning.

\end{abstract}

}

\maketitle

\IEEEdisplaynontitleabstractindextext

%
\IEEEpeerreviewmaketitle

\IEEEraisesectionheading{\section{Introduction}\label{sec:introduction}}

\IEEEPARstart{D}{istributed} stochastic gradient descent (SGD)\nocite{park2023s3gd} is a widely used method for solving large-scale optimization problems with data parallelism \cite{zinkevich2010parallelized, bottou2010large}. In theory, synchronous distributed SGD can linearly increase the training speed of high-dimensional models with the number of workers \cite{chen2016revisiting}. However, implementing such distributed SGD faces challenges, including high communication costs for exchanging gradient information between the central server and workers. This communication cost grows linearly with the number of workers. For example, the state-of-the-art image classification model, Coca \cite{yu2022coca}, has 2.1 billion parameters, requiring 8.4GB of information exchange per iteration for each worker and the central server, making it unaffordable for distributed training with limited communication networks. To address this communication bottleneck, communication-efficient distributed learning algorithms must be developed to minimize communication costs while maintaining high learning performance.

 In recent years, significant advancements have been made in the field of communication-efficient distributed learning algorithms. One popular approach to reducing communication costs per iteration is gradient quantization, such as QSGD \cite{alistarh2017qsgd}, where each worker quantizes the locally computed stochastic gradient to a limited number of bits. Another example is signSGD, a simple yet effective algorithm that only sends the signs of the locally computed gradient vector to workers \cite{bernstein2018asignsgd}. This one-bit quantized information surprisingly has a theoretical convergence guarantee for optimizing a broad range of non-convex loss functions. SignSGD has been shown to achieve the same convergence speed as traditional SGD with full gradient resolution.

Gradient sparsification is another approach to communication-efficient distributed learning, where each worker selects the largest components of the gradient and sends them to the server. This method has been shown to reduce communication costs by one to two orders of magnitude compared to QSGD and still holds the theoretical guarantees for optimizing non-convex loss functions \cite{stich2018sparsified}. Recent developments include combining gradient sparsification with quantization to further reduce communication costs  \cite{basu2019qsparse, wen2017terngrad, sattler2019robust, renggli2019sparcml}.

 In this paper, continuing the same spirit to attain a synergistic gain of sparsification and quantization, we put forth a distributed learning algorithm called Sparse-Sign-Stochastic gradient descent with the majority vote, ${\sf S}^3$GD-MV. Unlike the prior work \cite{bernstein2018asignsgd}, we show that updating a few selective components' signs of a local gradient vector is more beneficial than updating complete sign information of the gradient under majority voting aggregation in both the test accuracy and communication cost savings, provided that the sparsity level is sagaciously chosen with the number of workers and the model size. This observation is counter-intuitive because better test accuracy is achievable with lower communication costs than signSGD. We provide a clear explanation for this result by giving some examples. More importantly, we prove that this aggressively compressed gradient information update method with majority voting aggregation has a theoretical convergence guarantee for a broad class of non-convex loss functions.

\subsection{Related Works}
There has been a plethora of work in the theoretical literature aiming to study communication-efficient distributed learning.  Arguably, the prior work can be classified into three main categories: i) quantization, ii) sparsification, and iii) sparsification combined with quantization.

{\bf Quantization:} Gradient quantization effectively reduces the communication cost in distributed learning. Stochastic quantization techniques have received significant attention because they hold the unbiased property, i.e.,  the mean of the compressed stochastic gradient approximately equals the true mean \cite{alistarh2017qsgd, wen2017terngrad}. To aggressively compress the gradient, 1-bitSGD, SGD using one-bit quantization with the error compensation mechanism, was initially introduced by \cite{seide20141, strom2015scalable}. This scheme empirically showed that almost no loss in accuracy occurs compared to the full-precision counterpart while considerably reducing the cost. SignSGD with a majority vote is a different technique with 1-bitSGD in the quantization method \cite{pmlr-v38-carlson15, bernstein2018asignsgd}. Unlike 1-bitSGD, signSGD simply takes the sign of the stochastic gradient without the error compensation mechanism. This leads to a biased estimator of the stochastic gradient. Despite the biased gradient,  it was shown to provide a theoretical convergence guarantee while supporting compression in both communication directions between workers and a server. SignSGD also has some variational methods such as error-feedback case \cite{zheng2019communication, karimireddy2019error}. There exist many other quantization approaches: compressing the gradient difference \cite{mishchenko2019distributed, gorbunov2021marina}, vector quantization \cite{shlezinger2020uveqfed, gandikota2021vqsgd}, and adaptive quantization \cite{jhunjhunwala2021adaptive, mao2022communication}. The major hindrance to these one-bit quantization methods is that they can diminish the communication cost at most by a factor of 32. For instance, when optimizing Coca with 2.1 billion model parameters, each worker requires to send the one-bit quantized gradient vector with the size of 0.52 GB per iteration. Consequently, training extremely large-scale models using a low-bandwidth communication network may be insufficient.

{\bf Sparsification:} 
Gradient sparsification is another gradient compression technique that can drastically diminish the communication cost \cite{aji2017sparse, shi2019understanding}. The idea is that each worker chooses the $K$ largest components of a gradient in magnitude and sends them to the server in every iteration. This technique has been implemented with the error feedback or accumulation mechanism to ensure that all coordinates have opportunities to be updated as needed \cite{alistarh2018convergence, stich2018sparsified}. These have shown that the sparsified SGD provides the same convergence rate as vanilla SGD by a proper error accumulation method with memory. One drawback of the top-$K$ sparsification is the requirement of additional communication costs to encode the sparsity pattern along with the full-resolution information for the chosen $K$ parameter values. \cite{wangni2018gradient} suggests a stochastic sparsification method based on the variance of gradient, which empowers the top-$K$ sparsification. Based on the above methods, many modification methods have also been proposed: using momentum \cite{lin2017deep}, sketching method which finds heavy coordinates of the gradient \cite{rothchild2020fetchsgd, jiang2018sketchml}, and adaptive sparsification \cite{han2020adaptive, zhang2022mipd}. Meanwhile, several variants of the top-$K$ sparsification method have been proposed to reduce the implementation complexity. For instance, to avoid sorting operation, which is not friendly for GPU, another improved methods for top-$K$ sparsification have been presented to improve the learning performance with low complexity \cite{mahmoud1995analysis, shanbhag2018efficient, shi2019understanding, li2022near}.

{\bf Sparsification combined with quantization:}  The most relevant prior work to our paper is to harness the sparsification combined with quantization. \cite{basu2019qsparse} proposed a distributed learning algorithm called Qsparse-local-SGD. The idea of Qsparse-local-SGD is to jointly harness the gradient sparsification, quantization with the error compensation, and local computation. \cite{sattler2019robust, renggli2019sparcml} also suggests very similar learning methods with \cite{basu2019qsparse}, but they consider the downlink compression, which compresses the model update comes from the server. These algorithms are similar to our algorithm, but the one-bit quantization levels are computed using the empirical average of the non-zero values. Also, no proofs for the convergence rates of these algorithms were provided \cite{sattler2019robust, renggli2019sparcml}. Another related work is in \cite{wen2017terngrad}, where a gradient is aggressively quantized into three levels $\{-1,0,1\}$. This gradient compression technique, called TernGrad, uses probabilistic sparsification and quantization mechanisms to keep the unbiased expectation. This fact differs from our  top-$K$ sparsification and deterministic sign quantization method for gradient compression. In addition, our algorithm also takes a majority vote aggregation to reduce the communication cost from a server to workers, which was not considered in both Qsparse-local-SGD and TernGrad.


\subsection{Contributions}
We consider a generic distributed learning problem in which a distributed set of $M$ worker nodes independently perform computations to minimize an empirical-risk loss function with model size $N$ using locally stored data sets. 

\begin{itemize}
	\item We first put forth a distributed SGD composing top-$K$ sparsification and sign quantization along with majority vote aggregation called ${\sf S}^3$GD-MV. To the best of our knowledge, this is the first algorithm that jointly combines top-$K$ sparsification and na\"ive sign quantization along with  majority voting aggregation for distributed optimization. Thanks to the aggressive gradient quantization, ${\sf S}^3$GD-MV only requires exchanging at most  $M\left[K + K\log_2\left( \frac{N}{K} \right)+N\right]$ bits to the server per iteration. Since signSGD requires a communication cost of $2MN$, ${\sf S}^3$GD-MV ensures to attain a higher compression gain than signSGD, and the gain is pronounced by decreasing the sparsity level $K$.

	\item Our primary contribution is to provide the theoretical convergence guarantees for ${\sf S}^3$GD-MV  in optimizing a broad class of non-convex loss functions under some mild conditions. A rigorous theoretical analysis of the convergence rate for ${\sf S}^3$GD-MV is challenging because of the combined effect of sparsification, sign quantization, and majority vote aggregation. By generalizing the analytical tool developed in \cite{bernstein2018asignsgd}, we establish the convergence rate in terms of critical algorithmic parameters, chiefly sparsity level $K$,  the number of workers $M$, and model size $N$. Our key finding is that the convergence rate can be identical in order compared to the vanilla signSGD even with much lower communication cost, provided that the sparsity level is properly chosen with the number of workers and the model size. Precisely, we derive the optimal sparsity level that minimizes the upper bound of the convergence rate is $\frac{K}{N} =cM^{\frac{3}{2}}$ for some positive constant $c$.  We interpret this counter-intuitive  result with the principle of ``\textit{the power of representative democracy}".


 \item We also provide the convergence rate for ${\sf S}^3$GD-MV with the random-$K$ specification. From this result, we analytically explain how the top-$K$ sparsification can improve the learning performance compared to the random-$K$ specification method. 
	
	\item From simulations, we verify the performance gain of the  ${\sf S}^3$GD-MV over the existing algorithms. We train CNN and ResNet-56 models using MNIST and CIFAR-10 datasets. We observe that, in a certain setup, ${\sf S}^3$GD-MV can remarkably diminish the communication cost by about 300x, 200x, and 10x compared to the vanilla SGD, top-$K$ SGD, and signSGD, while offering the same test accuracy.   

\item This is an extended version of our conference paper \cite{park2023s3gd}, providing detailed proofs for the convergence rate of our proposed algorithm, $\mathsf{S}^3$GD-MV. The proofs in Section \ref{sec:pfs} are not included in \cite{park2023s3gd}. We also note that no convergence rate proof of $\mathsf{S}^3$GD-MV using the random-$K$ sparsification is provided in \cite{park2023s3gd}. In addition, using a rich set of simulation results, we explain how the learning performance of $\mathsf{S}^3$GD-MV changes regarding data distributions, communication costs, sparsification operators, and hyper-parameters, which are not covered in \cite{park2023s3gd}.

\end{itemize}




\section{${\sf S}^3$GD-MV Algorithm}
In this section, we briefly explain a distributed learning problem and present the proposed  ${\sf S}^3$GD-MV algorithm with some motivating examples in the sequel.

\subsection{Preliminary}\label{sec:system-model}

Let ${\bf x}\in \mathbb{R}^N$ be a model parameter vector with size $N$.  The mathematical formulation of machine learning can be defined as the following optimization problem:
\begin{align}
    \min_{{\bf x} \in \mathbb{R}^N}  f({\bf x}) := \mathbb{E}_{\xi \sim \mathcal{D}} \left[ F({\bf      x};\xi) \right], \label{eq:opt1}
\end{align}
where $\xi$ is a random sample drawn from data set $\mathcal{D}$ and $F(\cdot)$ is the problem-specific empirical loss function, which can be either convex or non-convex. In a distributed leaning setting with $M$ workers, the optimization problem of \eqref{eq:opt1} boils down to \begin{align}
	   \min_{{\bf x} \in \mathbb{R}^N}  \left[ f({\bf x}):=\frac{1}{M}\sum_{m=1}^Mf_m({\bf x})  \right],
\end{align}
where 
$f_m({\bf x}) = \mathbb{E}_{\xi_m \sim \mathcal{D}_m} \left[ F_m({\bf x};\xi_m) \right]$ is the local loss function at worker $m\in [M]$, which is obtained from the random data sample $\xi_m$ from the portion of the data set assigned to worker $m$, $\mathcal{D}_m$.

  We introduce two operations for gradient quantization. The top-$K$ sparsification operator selects the $K$ largest components of vector ${\bf u}\in \mathbb{R}^N$  in magnitude. This operation can be defined using an indicator function with a threshold $\rho$ as
\begin{align} \label{eqn:topk}
	{\sf TopK}({\bf u}) = {\mathbbm 1}_{\{ |{\bf u}| \geq \rho \}},
	\end{align}
	where $\mathbbm{1}_{\mathcal{A}} =1$ if $\mathcal{A}$ is true. Otherwise, $\mathbbm{1}_{\mathcal{A}} =0$. The threshold $\rho$ can be chosen as an arbitrary value between the $K$th largest and the $(K+1)$th largest component of ${\bf u}$ in magnitude.   Composing this top-$K$ sparsification operator and the sign quantization, we define an operator $\mathsf{TopKSign}(\cdot): \mathbb{R}^N\rightarrow \left\{ -1, 0, +1 \right\}^N$ that maps ${\bf u}\in \mathbb{R}^N$ into an $N$-dimensional ternary vector with sparsity parameter $\gamma= \frac{K}{N} \in [0, 1]$, which is defined as
\begin{align}
	 \mathsf{TopKSign} \left( {\bf u} \right) = \mathsf{sgn} \left( {\sf TopK}({\bf u})  \right), \label{eq:topksign}
\end{align}	
where  $\mathsf{sgn}({\bf u})={\mathbbm 1}_{\{ {\bf u}>0 \}}-{\mathbbm 1}_{\{{\bf u}<0 \}}$. All operations are applied in an element-wise fashion.

\subsection{Algorithm}
\label{sec:s3gd}

\hspace{1.2em} {\bf Stochastic gradient computation:} For given model parameter at iteration $t$, denoted by ${\bf x}^t$, worker $m\in [M]$ computes the local gradient ${\bf \tilde g}^t_m:=\nabla F_m({\bf x}^t,\xi_m^t) \in \mathbb{R}^N$ where $\xi_m^t$ is a randomly sampled mini-batch of dataset $\mathcal{D}_m$ at iteration $t$. Notice that the locally computed gradient is a random vector because of the mini-batch data sampling.
 
{\bf Gradient update with sparification error compensation:} Worker $m\in [M]$ updates the stochastic gradient ${\bf \tilde g}^t_m$ by adding the sparsification errors accumulated in previous iterations.  $\mathbf{e}_m^{t}$, i.e., 
 \begin{align}
 	{\bf g}^t_m={\bf \tilde g}^t_m+\eta {\bf e}^t_m, \label{eq:gradupdate}
 \end{align}
where $\mathbf{e}_m^{t}$ is the error compensation added at iteration $t$, which contains the sum of the stochastic gradient components that the top-$K$ sparsificaiton has not selected during the previous iterations. Each worker updates this memory term at iteration $t$ using ${\bf g}^{t-1}_m$ as
 \begin{align}
 	\mathbf{e}_m^{t} = \mathbf{g}_m^{t-1} - \mathsf{TopK} \left( \mathbf{g}_m^{t-1} \right).
 \end{align}
 In \eqref{eq:gradupdate}, $\eta>0$ denotes a weight parameter that controls the effect of the error compensation term. The sparsification error term can be pronounced by increasing $\eta$, which allows the selection of the gradient components more uniformly over coordinates.  When $\eta=0$, no memory is needed, and ${\bf g}^t_m={\bf \tilde g}^t_m$.

\begin{algorithm}[tb]
   \caption{${\sf S}^3$GD-MV}
   \label{alg:s3gd}
\begin{algorithmic}
    \STATE \textbf{Input:} initial model $\mathbf{x}^0$, learning rate $\delta^t$, gradient sparsity $K$, the number of workers $M$, initial accumulated error $\mathbf{e}_m^0 = \mathbf{0}$, error weight $\eta$, total iteration $T$
    \FOR{$t=0:T-1$}
    \FOR{\textbf{worker} $m=1: M$}
        \STATE \textbf{compute} $\tilde{\mathbf{g}}_m^t$ (local gradient)
        \STATE $\mathbf{g}_m^t \leftarrow \tilde{\mathbf{g}}_m^t + \eta \mathbf{e}_m^t$
        \STATE $\mathbf{e}_m^{t+1} \leftarrow \mathbf{g}_m^t - \mathsf{TopK} \left( \mathbf{g}_m^t \right)$
        \STATE \textbf{send} $ \mathsf{TopKSign} \left( \mathbf{g}_m^t \right)$ to \textbf{server}
        \STATE \textbf{receive} $\mathsf{sgn} \! \left[ \sum_{m=1}^M \!\! \mathsf{TopKSign} \left( {\bf g}_m^t \right) \right]$ from \textbf{server}
        \STATE ${\bf x}^{t+1} \leftarrow {\bf x}^t - \delta^t \mathsf{sgn} \! \left[ \sum_{m=1}^M \!\! \mathsf{TopKSign} \left( {\bf g}_m^t \right) \right]$
    \ENDFOR
    \ENDFOR
\end{algorithmic}
\end{algorithm}

{\bf Gradient compression:}  Worker $m\in [M]$ performs the compression for ${\bf g}^t_m$ by using $\mathsf{TopKSign}\left({\bf g}^t_m\right) \in \{-1,0,1\}^N$ defined in \eqref{eq:topksign}. To select the $K$ largest component in magnitude, the threshold $\rho_{m}^t(K)$ is chosen between the $K$th largest and the $(K+1)$th largest component of ${\bf g}^t_m$ in magnitude. Then, the compressed vector $\mathsf{TopKSign}\left({\bf g}^t_m\right)$ is sent to the server via a communication network. Since $\|{\sf TopKSign}({\bf g}^t_m)\|_0=K$, the communication cost from each worker to the server is roughly $K + K\log_2\left( \frac{N}{K} \right)$ bits per communication round.

{\bf Majority vote :} The server performs a majority vote for gradient aggregation. Let ${\mathcal N}_m^t=\{n\in[N]: \mathsf{TopKSign} \left( g_{m,n}^t \right)\neq 0\}$ be the non-zero supports selected by worker $m$ at iteration $t$. Then, we can define a union of ${\mathcal N}_m^t$ as ${\mathcal N}^t=\bigcup_{m=1}^M{\mathcal N}_m^t$ and its complement set ${\mathcal N}_c^t =[N]/{\mathcal N}^t$. Then, the majority vote principle is applied over the union of the non-zero support sets $n\in {\mathcal N}^t$. Whereas, for $n \in {\mathcal N}_c^t$, the server retains the zero values and sends them back to all workers.  

{\bf Model update:} Each worker updates the model parameter with the learning rate  $\delta^t$ as
\begin{align} \label{eqn:s3gd}
    {\bf x}^{t+1} = {\bf x}^t - \delta^t \cdot \mathsf{sgn} \left[ \sum_{m=1}^M\mathsf{TopKSign} \left( {\bf g}_m^t \right) \right].  
\end{align}
 The communication cost from the server to the worker is at most $|{\mathcal N}^t| + |{\mathcal N}^t|\log_2\left( \frac{N}{|{\mathcal N}^t|} \right)\leq N$ bits per iteration.

The entire procedure is summarized in Algorithm \ref{alg:s3gd}.


 \subsection{Synergistic Gains of Sparsification and Sign Quantization: The Power of Representative Democracy}
 
 Why more aggressive compression can be better than just sign compression? To understand this, it is instructive to  consider some motivating examples that elucidate the synergistic benefits of harnessing both sparsification and sign quantization under majority voting aggregation. 
 
 {\bf Example 1:} Consider a case of $M=3$. Suppose that the $n$th component of ${\bf g}_{m}^t$ are assumed to be generated at iteration $t$ as $g_{1,n}^t=3$,  $g_{2,n}^t=-0.3$, and $g_{1,n}^t=-0.03$. If the three workers send them with full-precision to the server, the sign of the aggregated gradients becomes positive, i.e., ${\sf sign}\left(\frac{1}{3}\sum_{m=1}^3 g_{m,n}^t\right) = +1$. When signSGD is applied, the three workers send ${\sf sign}\left(g_{1,n}^t\right)=+1$, ${\sf sign}\left(g_{2,n}^t\right)=-1$, and ${\sf sign}\left(g_{3,n}^t\right)=-1$. Then,  according to the majority vote, the server aggregates the sign in the wrong direction with the true one because ${\sf sign}\left(\sum_{m=1}^3 {\sf sign}\left(g_{m,n}^t\right)\right) = -1$. This sign error can be recovered, provided each worker only participates the sign vote when it is confident on its value. Specifically, we impose that worker $m$ sends the sign information if $|g_{m,n}^t|>1$. Then, worker 1 can only send the sign to the server ${\sf sign}\left(g_{1,n}^t\right)=+1$, while the remaining two workers do not participate in the vote. As a result, the aggregated sign becomes ${\sf sign}  \left({\sf sign}\left(g_{1,n}^t\right)\right) = +1$, which aligns with the true sign of the aggregated gradients, correcting the sign error of the signSGD. This example clearly reveals that the sparsification with the $K$ largest components in magnitude before the sign quantization is beneficial in decoding the true sign of the aggregated gradients. This principle can be interpreted as ``\textit{the power of representative democracy}". This is clearly different from the signSGD using the principle of ``\textit{the power of democracy}".

 {\bf Example 2:}  From Example 1, the synergistic gain of the sparsification and quantization is manifest. However, too much sparsification may lose the gain. We explain this critical problem with a lens through a celebrated  \textit{``balls and bins problem''}. Suppose there are $M$ workers, and each worker sends only the maximum component in magnitude with the sign quantization. The selected components by $M$ workers are assumed to be independently and uniformly distributed in $\{1,2,\ldots, N\}$, where $N$ is model size. From the balls and bins principle,  given $M$ workers, the probability that worker $m$ participates to update the $n$th component is
 \begin{align}
 	\mathbb{P}\left[ m~ {\rm worker} \rightarrow n{\rm th}~~{\rm component}\right] =\frac{1}{N}.
 \end{align}
 As a result, the probability that the $n$th gradient component is not updated (i.e., the empty bin) becomes
  \begin{align}
 	\mathbb{P}\left[  n{\rm th}~~{\rm component}~{\rm is}~{\rm empty}\right] =\left(1-\frac{1}{N}\right)^M\simeq e^{-\frac{M}{N}}.
 \end{align}
This effect explains why too much sparsification can lead to a significant loss in learning performance because the many coordinates are not likely to be updated in each iteration. To remedy this issue, we need to scale up the sparsity level $K$ to ensure that at least one or more workers can participate in every coordinate of the gradient with a high probability. Suppose each worker selects $K$ coordinates uniformly over $\{1,2,\ldots, N\}$. Then, the probability that the $n$th gradient component is empty is given by
\begin{align}
    \mathbb{P}\left[  n{\rm th}~~{\rm component}~{\rm is}~{\rm empty}\right] =\left(1-\frac{K}{N}\right)^{M}\simeq e^{-\frac{KM}{N}}.
\end{align} 
As a result, the algorithm should carefully choose the sparsity level $K$ according to the number of workers $M$ so that all the coordinates are sufficiently scheduled per training round, such as $M^{\ell} K = cN$ for some constant $c > 0$ and $\ell \leq 1$.


%
%

\section{Convergence Rate Analysis}
\label{sec:analysis}

In this section, we analyze a convergence rate of $\mathsf{S}^3$GD-MV in the non-convex loss setting.  We commence by explaining some relevant assumptions to derive our analytical result. Then, we introduce our main result on the convergence rate.



%


\subsection{Assumptions}

\begin{assumption}[Lower bound]
\label{ass:1}
 For all ${\bf x} \in \mathbb{R}^N$ and some local minimum points ${\bf x}^\star$, we have an objective value as
\begin{align}
    f({\bf x}) \geq f\left( {\bf x}^\star \right) = f^\star.
\end{align}
\end{assumption}
This assumption is necessary for the convergence to local minimum points.

\begin{assumption}[Coordinate-wise smoothness]
\label{ass:2}
For all ${\bf x}, {\bf y} \in \mathbb{R}^N$, there exists a vector with non-negative constants $\mathbf{L} = \left[ L_1, \cdots, L_N \right]$ that satisfies
\begin{align} \label{eqn:assumption2}
    \left| f({\bf y}) \! - \! f({\bf x}) \! - \! \left\langle \nabla f({\bf x}), {\bf y} \! - \! {\bf x} \right\rangle \right| \leq \sum_{n=1}^N \frac{L_n}{2} \left( y_n \! - \! x_n \right)^2.
\end{align}
\end{assumption}
This assumption indicates that the objective function holds coordinate-wise Lipschitz condition.

\begin{assumption}[Unbiased stochastic gradient with finite variance]
\label{ass:3}
The stochastic local gradient with the error compensation at iteration $t$, ${\bf g}_m^t$, is unbiased and each component of ${\bf g}_m^t$ has a finite variance bound with a non-negative constant $\boldsymbol{\sigma} = \left[ \sigma_1, \cdots, \sigma_N \right]$, i.e., $\forall m \in [M], \forall n \in [N]$,
\begin{align}
    \mathbb{E} \left[ {\bf g}_m^t \right] = {\bf \bar g}_m^t, ~~~~\mathbb{E} \left[ \left( g_{m, n}^t - {\bar g}_{m,n}^t \right)^2 \right] \leq \sigma_n^2.
\end{align}
\end{assumption}

\begin{assumption}[Uniform sampling of the non-zero supports]
\label{ass:4}
$\forall m \in [M], \forall n \in [N]$, the top-$K$ operator with the error compensation uniformly selects the non-zero supports at iteration $t\in [T]$, i.e.,  
\begin{align} \label{eqn:assumption4}
    \mathbb{P} \left[ \mathsf{TopKSign} \left( {g}_{m,n}^t  \right)  \ne 0 \right] = \gamma,
\end{align}
 where $\gamma =\frac{K}{N}$. 

\end{assumption}
This assumption indicates that the top-$K$ selection with the error compensation technique allows updating every coordinate in a round-robin fashion to ensure uniform sampling. In particular, properly tuning the memory weight parameter $\eta$ makes this uniform sampling assumption more accurate. We will justify this assumption via simulation results in Section \ref{sec:simulation}. We capitalize that this assumption help the convergence proof to be more mathematically tractable. Without this uniform sampling assumption, i.e., with non-uniform sampling, the proposed algorithm can still achieve a higher learning performance than the existing algorithms even with reduced communication costs, which will be validated in Section \ref{sec:simulation}.




\subsection{Convergence Rate of $\mathsf{S}^3$GD-MV}
\label{sec:rate}

Armed with our assumptions above, we present our theoretical results. Before presenting the main result, we introduce some useful lemmas that are needed to prove the convergence rate result. Then, we provide some interpretation of the convergence rate by comparing it with the convergence rate of signSGD \cite{bernstein2018asignsgd} in the sequel.  

\vspace{0.2cm}

\begin{lemma} \label{lem:1}
The top-$K$ threshold parameter $\rho_{m,n}^t(\gamma)$ for $m\in [M]$ and $n\in [N]$ is lower bounded by 
 \begin{align}
    \rho_{m,n}^t(\gamma) \ge \frac{\epsilon}{\sqrt{\gamma}} \left| \bar{g}_{n}^t \right|. \label{eq:lem1}
\end{align}
for some $\epsilon\geq0$.
\end{lemma}

\begin{proof}
    See Section \ref{pf:lem:1}.
\end{proof}

This lemma indicates that the threshold parameter $\rho_{m,n}^t(K)$ can be represented as a polynomial function of sparsity level parameter $\gamma$ with the magnitude of the true gradient mean $\left| \bar{g}_{n}^t \right|$. Based on Assumption \ref{ass:4} that gradient components are uniformly selected by $\mathsf{TopKSign}$, we can derive the convergence rate with element-wise threshold $\rho_{m,n}^t(\gamma)$. This threshold increases with the mean of true gradient $\left| \bar{g}_{n}^t \right|$, while it is inversely proportional to the sparsity level $\gamma=\frac{K}{N}$. For instance, if we choose a large $K$, the threshold diminishes accordingly, which agrees with our intuition. This lemma is a stepping stone toward establishing the sign decoding error probability bound, which is stated in the following lemma.

\begin{lemma} \label{lem:2}
When $M=1$, the sign flipping error by ${\sf S}^3$GD-MV is upper bounded by
\begin{align*}
    & \mathbb{P} \left[ \mathsf{TopKSign} \left( g_{m,n}^t \right) \ne \mathsf{sign} \left( \bar{g}_{n}^t \right) \right] \\
    & \hspace{10em} \leq  \dfrac{1}{\sqrt{B^t}\left(\! 1\!+\!\frac{\epsilon}{\sqrt{\gamma}} \!\right)}  \frac{\sigma_n}{\left| \bar{g}_{n}^t \right|}, \numberthis{} \label{eq:lem2}
\end{align*}
where $B^t$ is the mini-batch size at iteration $t$.
\end{lemma}   
\begin{proof}
    See Section \ref{pf:lem:2}.
\end{proof}

Lemma \ref{lem:2} demonstrates that the sign flipping error can be improved by either increasing the mini-batch size $B^t$ or by decreasing the sparsity level $\gamma$ (i.e., a small $K$). In addition, increasing signal-to-noise ratio (SNR), $\frac{\left| \bar{g}_{n}^t \right|^2}{\sigma_n^2 }$, makes less decoding errors. We highlight that our decoding error bound is a generalization of the result in \cite{bernstein2018asignsgd} by considering the top-$K$ selection operation on the top of sign quantization.  
   
%
%
    
    
\begin{lemma}\label{lem:3}
    Let $\mathcal{M}_n^t$ be a collection of the workers who send the $n$th component of the stochastic gradients at iteration $t$, which is defined as
\begin{align}
    \mathcal{M}_n^t = \left\{ m\in [M]:  \mathsf{TopKSign} \left( g_{m,n}^t \right) \neq 0\right\}
\end{align} 
with cardinality $M_n^t \triangleq | \mathcal{M}_n^t |$. Then, $M_n^t$ follows the binomial distribution ${\sf B}(M,\gamma)$, i.e., 
\begin{align}
    \mathbb{P} \left[ M_n^t = u \right] = \binom{M}{u} \gamma^u (1-\gamma)^{M-u}.
\end{align} 
for $u\in \{0,1,\ldots, M\}$.
\end{lemma}

\begin{proof}
    From Assumption \ref{ass:4}, each worker independently samples each coordinate with probability $\gamma=\frac{K}{N}$. Consequently, $M_n^t$ follows the binomial distribution ${\sf B}(M,\gamma)$.
\end{proof}

\begin{lemma} \label{lem:4}
Let ${p}_{m,n}^t =\mathbb{P} \left[ \mathsf{TopKSign} \left( g_{m,n}^t \right) \ne \mathsf{sign} \left( \bar{g}_{n}^t \right) \right]$. Then, conditioned on $M_n^t=U$, the sign decoding error of ${\sf S}^3$GD-MV at iteration $t$ is upper bounded as
\begin{align}
    &\mathbb{P} \left[ \left. {\sf sgn} \! \left( \sum_{m \in \mathcal{M}_n^t}  {\sf TopKSign} \left( g_{m,n}^t \right) \right) \neq  {\sf sign} \! \left( {\bar g}_{n}^t \right) \right| M_n^t = U \right] \nonumber\\
    & \leq  \left[4(1-{p}_{m,n}^t){p}_{m,n}^t \right]^{\frac{U}{2}}.
\end{align}
\end{lemma}

\begin{proof}
    See Section \ref{pf:lem:4}.
\end{proof} 

This lemma shows that the sign decoding error can be reduced, provided that the number of workers who select the same coordinate increases exponentially. This gain improves by increasing the number of workers $M_n^t$. In addition, it turns out that the decoding error decreases by the top-$K$ sparsification because the sign flip probability, ${p}_{m,n}^t$, reduces as decreasing $\gamma$ as proven in Lemma \ref{lem:2}. As a result, there exists an interesting trade-off between ${p}_{m,n}^t$ and $M_n^t$ with respective to $K$ ($\gamma$). Specifically, increasing $K$ improves the coding gain by providing a more chance of having a large $M_n^t$. Whereas, increasing $K$ also diminishes ${p}_{m,n}^t$ because $\rho_{m,n}^t(K)$ becomes smaller. In conclusion, we must carefully choose $K$ to enjoy both benefits. This analytical result of the sign decoding error perfectly matches our high-level explanation of how to choose $K$ to attain the synergistic gain in Examples 1 and 2.
   
Leveraging the lemmas mentioned above, we are ready to present our main result, which is stated in the following theorem. 
    
\begin{theorem}[Non-convex convergence rate of mini-batch $\mathsf{S}^3$GD-MV]
\label{thm:1}
 
 Given learning rate $\delta^t = \frac{1}{\sqrt{T \lVert \mathbf{L} \rVert_1}}$, batch size $B^t = T$, and a constant $0 \leq \epsilon \leq 1$, $\mathsf{S}^3$GD-MV converges as
\begin{align}  
    & \mathbb{E} \left[ \frac{1}{T} \! \sum_{t=0}^{T-1} \lVert \bar{\mathbf{g}}^t \rVert_1 \right]  \nonumber\\
    &\leq \! \frac{1}{\sqrt{T}} \! \left[ \! \sqrt{\lVert \mathbf{L} \rVert_1} \! \left( \! \frac{f^0 \! - \! f^\star }{\alpha(M,\gamma)}  + \! \frac{1}{2} \right) \! + \! \frac{\beta(M,\gamma)}{  \alpha(M,\gamma)}   \frac{2}{ 1\!+\!\frac{\epsilon}{\sqrt{\gamma}}} \lVert \boldsymbol{\sigma} \rVert_1 \right],  \label{eq:thm1}
\end{align}
where $\gamma=\frac{K}{N}$,  $
	\alpha(M,\gamma) = 1 - (1-\gamma)^M$
and
\begin{align}
	\beta(M,\gamma) =\sum_{u=1}^M \frac{1}{\sqrt{u}} \binom{M}{u} \gamma^u (1-\gamma)^{M-u}.
\end{align}

\end{theorem}

\begin{proof}
    See Section \ref{pf:thm:1}.
\end{proof}

\subsection{Special Cases}
\label{sec:int}
Our convergence rate analysis demonstrates that $\mathsf{S}^3$GD-MV guarantees to converge a stationary point as increasing the number of iterations $T$ with rate in order $\mathcal{O}\left(\frac{1}{\sqrt{T}}\right)$.  To better understand the result, we further provide several remarks on Theorem \ref{thm:1} by specializing important algorithmic parameters. 
\subsubsection{No sparsificiton case ($\gamma=1$)}
 When $\gamma=1$, i.e., no sparsificiton is applied for compression $N=K$, $\alpha(M,1)=1$ and $\beta(M, 1) =\frac{1}{\sqrt{M}}$. In this case, our algorithm is equivalent to signSGD, and the derived convergence rate with $\epsilon = 0$ coincides with the rate expression in \cite{bernstein2018asignsgd}. Consequently,  our result generalizes the convergence rate analysis of the prior work  \cite{bernstein2018asignsgd} in terms of the sparsification parameter $0<\gamma<1$.

\subsubsection{Single worker case ($M=1$)}

 When there is only a single worker ($M=1$), we observe different behavior in the constant terms of $\|{\bf L}\|_1$ and $\lVert \boldsymbol{\sigma} \rVert_1$. The $\lVert \mathbf{L} \rVert_1$ term increases with a rate of $\mathcal{O} \left( \frac{N}{K} \right)$, while the $\lVert \boldsymbol{\sigma} \rVert_1$ term changes at a rate of $\mathcal{O} \left( \sqrt{\frac{K}{N}} \right)$. This observation suggests that there may be an optimal sparsity level that results in the highest convergence rate.

\subsubsection{Large $M$ but small $\gamma$}
The most significant scenario occurs when we select a small value for $\gamma$ and a sufficiently large value for $M$. In this situation, the parameters become asymptotically $\alpha(M,\gamma) \approx M\gamma =\frac{MK}{N}$ and $\beta(M,\gamma) \approx M\gamma=\frac{MK}{N}$. These approximations result in the constant terms of $\|{\bf L}\|_1$ and $\lVert \boldsymbol{\sigma} \rVert_1$ having order $\mathcal{O} \left( \frac{N}{MK} \right)$ and $\mathcal{O} \left( \sqrt{\frac{K}{N}} \right)$, respectively. 

In this case, the sparsity parameter $\gamma=\frac{K}{N}$ plays an important role in altering the convergence rate. To speed up the convergence rate, we characterize the optimal sparsity parameter $\gamma$ that minimizes the upper bound of the convergence rate in Theorem \ref{thm:1}, which is stated in the following corollary. 

\begin{corollary}[Optimal sparsity of $\mathsf{S}^3$GD-MV]
\label{cor:1}
For small enough $\gamma \ll 1$, the optimal sparsity level that minimizes the upper bound of the convergence rate is 
\begin{align}
    \gamma^\star = \left( \frac{\epsilon \left( f^0 - f^\star \right)}{M} \frac{\sqrt{\lVert \mathbf{L} \rVert_1}}{\lVert \boldsymbol{\sigma} \rVert_1} \right)^{\frac{2}{3}}.
\end{align}
\end{corollary}

\begin{proof}
    See Section \ref{pf:cor:1}.
\end{proof}

This implies that to improve the convergence rate, we need to choose the smallest possible value of $K$ while keeping $M^{\frac{2}{3}}K \simeq N$. This aligns with our explanation in Example 2, where the algorithm must choose a large enough value of $K$ to update all components of the gradient in each iteration with high probability to attain the synergistic gain. Table \ref{table1} summarizes a comparison of the convergence rates with signSGD.

\begin{table}[t] 
\caption{Convergence rate of signSGD and ${\sf S}^3$GD-MV $(\gamma \ll 1)$}
\label{table1}
\begin{center}
\begin{small}
\begin{sc}
\begin{tabular}{P{0.15\columnwidth} P{0.15\columnwidth} P{0.15\columnwidth} P{0.15\columnwidth} P{0.15\columnwidth}}
\toprule
\multirow{2}{*}{Workers} & \multicolumn{2}{c}{signSGD} & \multicolumn{2}{c}{$\mathsf{S}^3$GD-MV} \\
\cline{2-5}
 & $\lVert \mathbf{L} \rVert_1$ & $\lVert \boldsymbol{\sigma} \rVert_1$ & $\lVert \mathbf{L} \rVert_1$ & $\lVert \boldsymbol{\sigma} \rVert_1$ \\
\midrule
\multirow{2}{*}{$M=1$} & \multirow{2}{*}{$\mathcal{O} (1)$} & \multirow{2}{*}{$\mathcal{O} (1)$} & \multirow{2}{*}{$\mathcal{O} \left( \frac{N}{K} \right)$} & \multirow{2}{*}{$\mathcal{O} \left( \sqrt{\frac{K}{N}} \right)$} \\ \\
\multirow{2}{*}{$M \geq 2$} & \multirow{2}{*}{$\mathcal{O} (1)$} & \multirow{2}{*}{$\mathcal{O} \left( \frac{1}{\sqrt{M}} \right)$} & \multirow{2}{*}{$\mathcal{O} \left( \frac{N}{MK} \right)$} & \multirow{2}{*}{$\mathcal{O} \left( \sqrt{\frac{K}{N}} \right)$} \\ \\
\bottomrule
& & & \multicolumn{2}{r}{$\left( \text{rate:} \times \frac{1}{\sqrt{T}} \right)$}

\end{tabular}
\end{sc}
\end{small}
\end{center}
\vskip -0.1in
\end{table}

\subsection{Effect of the random-$K$ sparsificaiton}

It is also instructive to compare the convergence rate when using the random-$K$ sparisifciation technique composed of the one-bit quantization instead of {\sf TopKSign}.  We define this operation by {\sf RandKSign}, which is
\begin{align}
   \mathsf{RandKSign} \left( \mathbf{g}_m^t \right) = \mathsf{sgn} \left( \mathsf{RandK} \left( \mathbf{g}_m^t \right) \right).
\end{align}
By replacing {\sf TopKSign} operator into {\sf RandKSign} in Algorithm \ref{alg:s3gd}, we can modify the ${\sf S}^3$GD-MV algorithm. The key distinction is that this modified algorithm does not require the error accumulation process in the selection of the $K$ components of the stochastic gradient. The entire steps are summerized Algorithm \ref{alg:s3gd_rand}.



The following theorem states the convergence rate of such modified algorithm.

%

\begin{algorithm}[tb]
  \caption{${\sf S}^3$GD-MV (Random-$K$ ver.)}
  \label{alg:s3gd_rand}
\begin{algorithmic}
   \STATE \textbf{Input:} initial model $\mathbf{x}^0$, learning rate $\delta^t$, gradient sparsity $K$, the number of workers $M$, total iteration $T$
   \FOR{$t=0:T-1$}
   \FOR{\textbf{worker} $m=1: M$}
        \STATE \textbf{compute} $\mathbf{g}_m^t$ (local gradient)
        \STATE \textbf{send} $\mathsf{RandKSign} \left( \mathbf{g}_m^t \right)$ to \textbf{server}
        \STATE \textbf{receive} $\mathsf{sgn} \! \left[ \sum_{m=1}^M \!\! \mathsf{RandKSign} \! \left( {\bf g}_m^t \right) \right]$ from \textbf{server}
        \STATE ${\bf x}^{t+1} \leftarrow {\bf x}^t - \delta^t \mathsf{sgn} \! \left[ \sum_{m=1}^M \!\! \mathsf{RandKSign} \left( {\bf g}_m^t \right) \right]$
   \ENDFOR
   \ENDFOR
\end{algorithmic}
\end{algorithm}

\begin{theorem} \label{thm:2}
    Given learning rate $\delta^t = \frac{1}{\sqrt{T \lVert \mathbf{L} \rVert_1}}$, and batch size $B^t = T$, $\mathsf{S}^3$GD-MV with $\mathsf{RandKSign}$ operator converges as
    \begin{align*}
        &\mathbb{E} \left[ \frac{1}{T} \sum_{t=0}^{T-1} \lVert \bar{\mathbf{g}}^t \rVert_1 \right] \\
        &\leq \frac{1}{\sqrt{T}} \left[ \sqrt{\lVert \mathbf{L} \rVert_1} \left( \frac{f^0 - f^\star}{\alpha (M, \gamma)} + \frac{1}{2} \right) + 2 \frac{\beta(M, \gamma)}{\alpha (M, \gamma)} \lVert \boldsymbol{\sigma} \rVert_1 \right]. \numberthis{}
    \end{align*}
 
\end{theorem}

\begin{proof}
    See Section \ref{pf:thm:2}.
\end{proof}

The key difference from the convergence rate in Theorem \ref{thm:1}  is the lack of the term $ \frac{2}{ 1\!+\!\frac{\epsilon}{\sqrt{\gamma}}} $ in $\lVert \boldsymbol{\sigma} \rVert_1$. Therefore, it deteriorates the learning performance because $\lVert \boldsymbol{\sigma} \rVert_1$ does not scale with $K$. This result indicates that it is very important to select the reliable $K$ components in the gradient sparsification in order to attain the synergistic gain of the sparsification and sign quantization.  We will validate this learning performance loss using simulation results in Section \ref{sec:simulation}.

\section{Simulation Results}
\label{sec:simulation}

In this section, we provide simulation results to validate the performance of $\mathsf{S}^3$GD-MV on broadly used benchmark datasets. We examine the performance in terms of test accuracy and communication costs and compare it with the existing communication-efficient distributed learning algorithms. In the sequel, we present the experimental settings and the results.

\begin{table}[t] 
\caption{Distributed learning algorithms}
\centering
\label{tab:alg}
\begin{small}
\begin{tabular}{P{0.3\columnwidth} P{0.6\columnwidth} }
\toprule
Algorithms & Total communication costs \\
\midrule
Vanilla SGD \cite{zinkevich2010parallelized} & $ [32MN + 32MN]\times T$ \\
Top-$K$ SGD with memory \cite{stich2018sparsified} & \multirow{2}*{$  \left(M\left[32K + K\log_2\left( \frac{N}{K} \right)\right] + 32MN\right)\times T$} \\
signSGD-MV \cite{bernstein2018asignsgd} & $(MN + MN)\times T$ \\
$\mathsf{S}^3$GD-MV & $ \left(M\left[K + K\log_2\left( \frac{N}{K} \right)\right] + MN\right)\times T$ \\
\bottomrule
\end{tabular}
\end{small} 
\vspace{-0.1cm}
\end{table}

\subsection{Simulation Environment}
\label{sec:settings}

%


{\bf Datasets and distributions:} We use two basic benchmark datasets, MNIST and CIFAR10  \cite{lecun1998gradient, krizhevsky2009learning}, which are broadly used for classification experiments. MNIST and CIFAR10 datasets contain 60,000 and 50,000 training images for each, and both datasets have 10,000 test image samples with ten classes.While data augmentation methods such as random cropping and horizontal random flipping are applied to CIFAR10, no data augmentation method is used for MNIST in the experiments. 

We consider two different data distributions. We randomly assign the data samples to workers using independent and identically distributed (IID) data distribution. In contrast, for the non-IID dataset, we allocate a single class of data to a particular worker; thereby, the data distributions across the workers are heterogeneous. 

 {\bf Models:} To validate the performance of the proposed algorithm, we consider two popular neural network models, convolutional neural network (CNN) \cite{lecun1998gradient} and ResNet-56  \cite{he2016deep}. CNN is a shallow network composed of five convolutional layers and three fully connected layers. The model size of CNN is set to be $N=5 \times 10^5$. Meanwhile, ResNet-56 with model size $N=8.5 \times 10^5$ is a typical network model for evaluating the performance on large datasets.  
    
 {\bf Benchmark:} We consider the four benchmark distributed learning algorithms:

\begin{figure}[t]
    \centerline{
    \subfloat[Convolutional layer in CNN]{%
        \includegraphics[width=0.5\columnwidth]{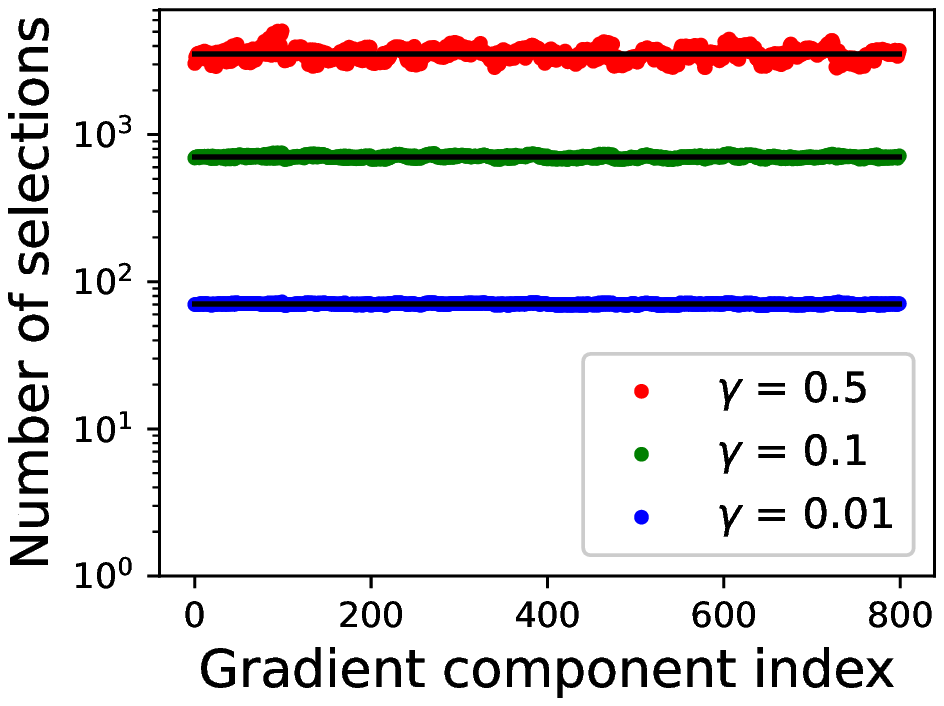}}
    \hfill
    \subfloat[Weight layer in CNN]{%
        \includegraphics[width=0.5\columnwidth]{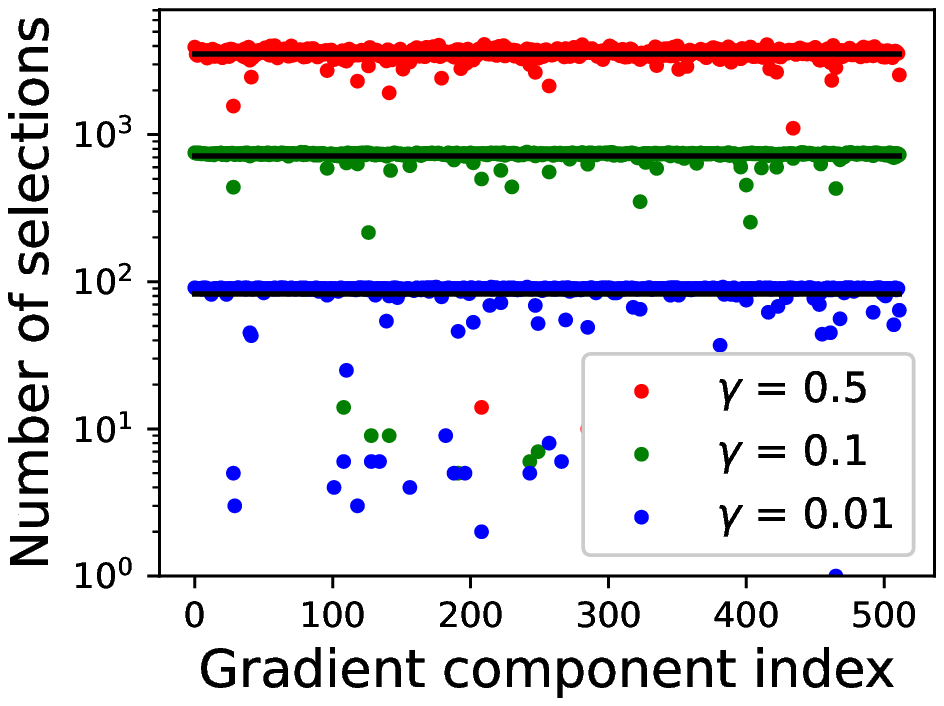}}
    }\vfill
    \centerline{
    \subfloat[Convolutional layer in ResNet]{%
        \includegraphics[width=0.5\columnwidth]{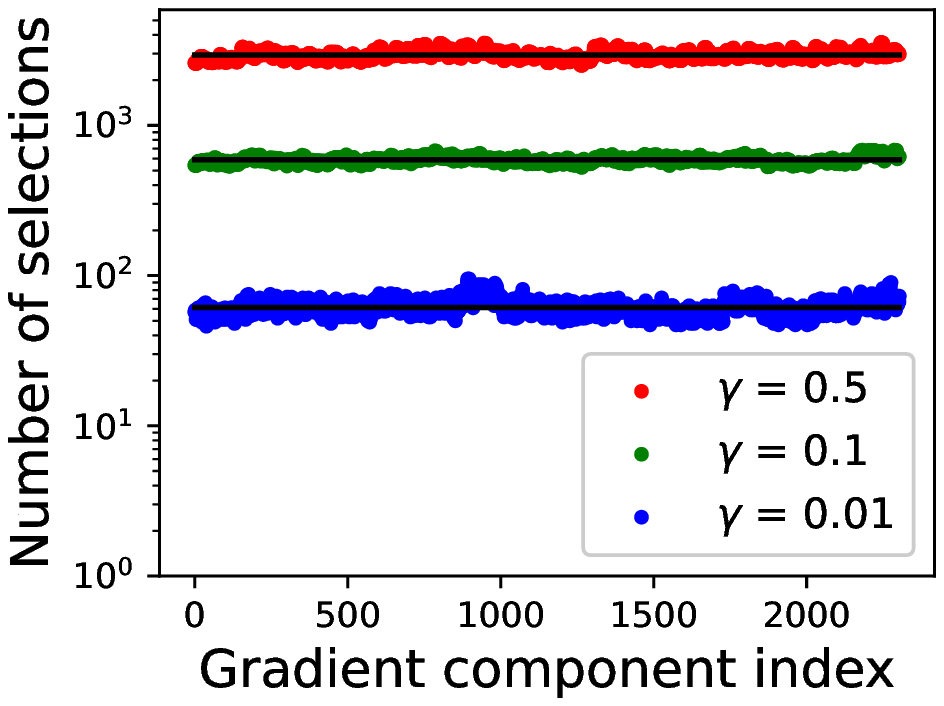}}
    \hfill
    \subfloat[Weight layer in ResNet]{
        \includegraphics[width=0.5\columnwidth]{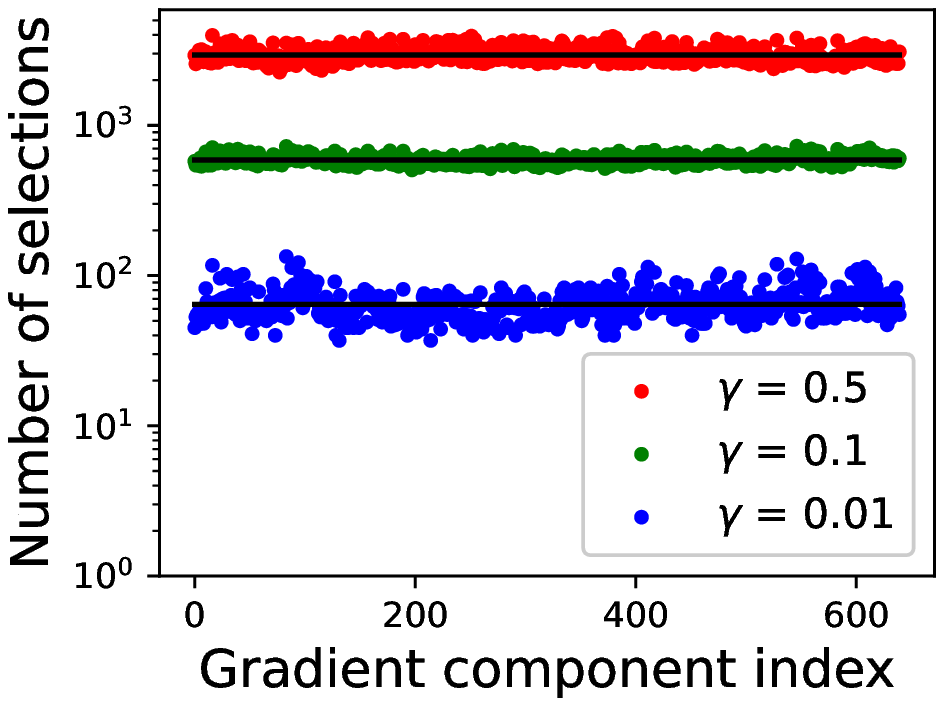}}
    }
    \caption{
    The number of selections of each gradient component on a certain layer of CNN (top) used for MNIST and ResNet (bottom) used for CIFAR10.
    }\label{fig:assumption4}
\end{figure}

  \begin{itemize}
 	\item Vanilla SGD  \cite{zinkevich2010parallelized}: this is a baseline distributed learning algorithm. Each worker and the server exchange the stochastic gradient with 32-bit resolution per iteration. The communication cost per round of the vanilla SGD, therefore, is $ 2\times (32NM)$ bits.  

 	\item Top-$K$ SGD with memory  \cite{stich2018sparsified}: the communication cost of this algorithm is at most $M\left[32K + K\log_2\left( \frac{N}{K} \right)\right] + 32NM$ bits. This algorithm attains the compression gain compared to the SGD for the gradient update from workers to the server.   
 	
 	\item signSGD with MV  \cite{bernstein2018asignsgd}: this algorithm requires the communication cost of $2MN$ bits, which is 32x reduction compared to that of the vanilla SGD.  
 	
 	\item $\mathsf{S}^3$GD-MV: our algorithm needs at most $M\left[K + K\log_2\left( \frac{N}{K} \right)\right] + MN$ per training round. It obtains the gradient compression gain from both the sparsification and one-bit quantization.

 	\end{itemize}
  To optimize the learning performance, we carefully choose the learning rate $\delta^t\in \{10^{-1},10^{-3}\}$ for the first two and the last two algorithms, respectively. The effects of the error weight and momentum will be investigated.


%

\subsection{Validation of uniform sampling assumption}


To validate our assumption of uniform sampling of the non-zero supports in the top-$K$ sparsification (Assumption \ref{ass:4}), we present numerical results. As shown in Fig. \ref{fig:assumption4}, we track the number of counts selected by the top-$K$ sparsification over 30 epochs for each gradient component index, which serves as the empirical probability mass function. Our numerical results indicate that the top-$K$ sparsification with a carefully selected error weight parameter $\eta$ allows for holding a uniform sampling property of the non-zero supports. This result supports  Assumption \ref{ass:4}.




\begin{figure}[!t]
    \centerline{
    \subfloat[$M=10$ and MNIST dataset]{%
        \includegraphics[width=0.5\columnwidth]{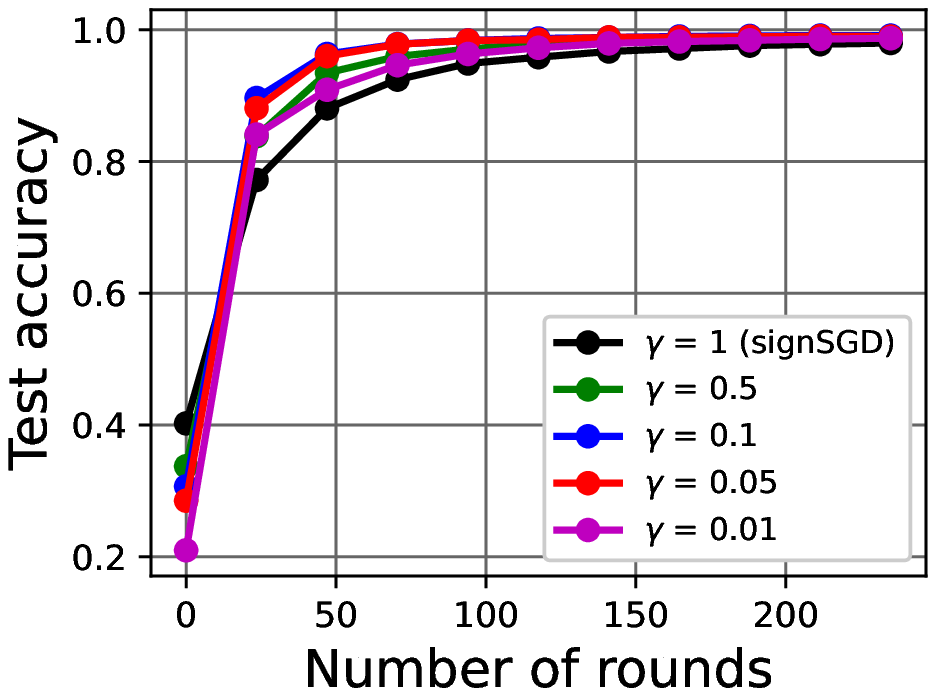}}
    \hfill
    \subfloat[$M=100$ and MNIST dataset]{%
        \includegraphics[width=0.5\columnwidth]{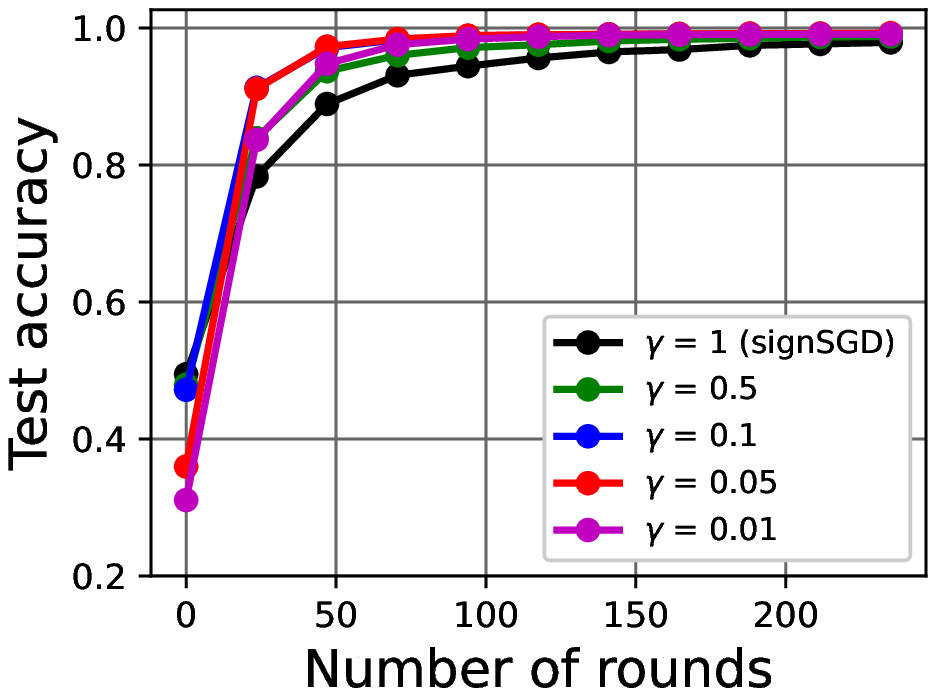}}
    
    }
    \vfill
    \centerline{
    \subfloat[$M=5$ and CIFAR10 dataset]{%
        \includegraphics[width=0.5\columnwidth]{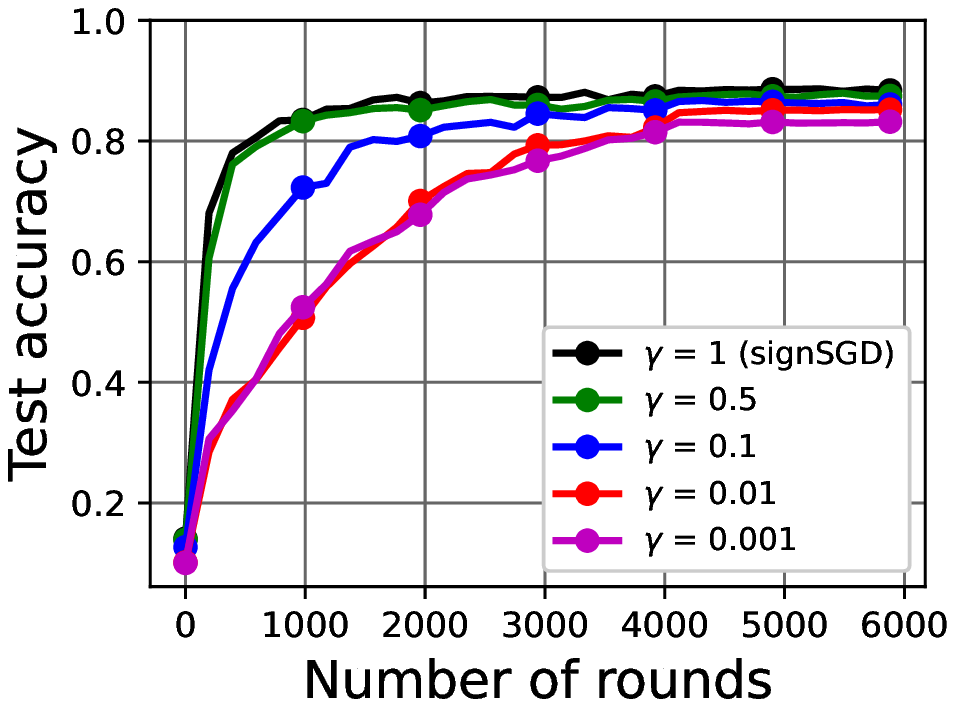}}
    \hfill
    \subfloat[$M=50$ and CIFAR10 dataset]{
        \includegraphics[width=0.5\columnwidth]{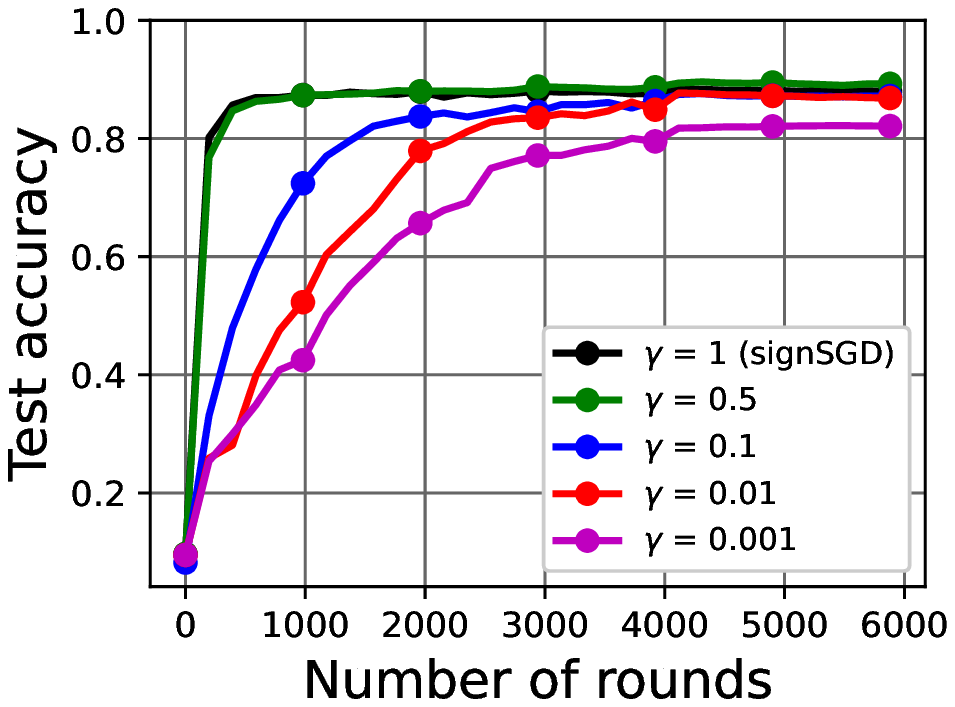}}
    }
    \caption{
    The test accuracy versa training rounds with the different number of workers $M$ for IID data distribution. 
    }\label{fig:s3gd_IID} 
    \vspace{-0.3cm}
\end{figure}

\subsection{Effects of $\gamma$ and $M$ for IID dataset }
 

The impact of sparsity and the number of workers on test accuracy is depicted in Fig. \ref{fig:s3gd_IID}. The figures on the top show the test accuracy of $\mathsf{S}^3$GD-MV with $M\in\{10,100\}$ for optimizing the CNN using the IID MNIST dataset. SignSGD is used as a benchmark when $\gamma=1$. For the MNIST dataset, our algorithm can outperform signSGD for all $\gamma$s and $M$s while saving the costs considerably. Interestingly, we can speed up the convergence rate by carefully choosing $\gamma=\frac{K}{N}$ such that $M^{\frac{2}{3}}K \simeq N$ as derived in Corollary \ref{cor:1}. For instance, as the number of workers increases, the optimal sparsity of $\mathsf{S}^3$GD-MV should be lower. The sparsity level of $\gamma=0.1$ and 0.05 for $M=10$ and 100 are shown to be optimal choices. This aligns with the finding in Corollary \ref{cor:1}.

The bottom figures illustrate the test accuracy of the proposed algorithm for the numbers of workers $M\in\{5,50\}$ when optimizing the ResNet-56 model using the CIFAR10 dataset. Since the CIFAR10 dataset is much complicated than the MNIST dataset, we need to use a higher $\gamma$ for ResNet-56 than those for the CNN for given $M$ when running the $\mathsf{S}^3$GD-MV algorithm to speed up the learning performance. As expected from Corollary \ref{cor:1}, it is possible to achieve the higher test accuracy when $\gamma=0.5$ and $M=50$, even better the test accuracy than signSGD. This result reveals that our algorithm improves the test accuracy and reduces the communication cost considerably. This is in line with our analysis in Theorem \ref{thm:1}.


\subsection{Effects of $\gamma$ and $M$ for non-IID dataset }

Fig. \ref{fig:s3gd_nonIID} demonstrates the effect of sparsity and worker number on the test accuracy for non-IID datasets. The test accuracy of ${\sf S}^3$GD-MV algorithm is plotted as the number of rounds is increased, for different values of sparsity parameter $\gamma$ and number of workers. The benchmark is signSGD, which is equivalent to our algorithm when $\gamma=1$.

As evident from Fig. \ref{fig:s3gd_nonIID}-(a) and \ref{fig:s3gd_nonIID}-(b), the traditional signSGD fails to train the CNN model for non-IID MNIST dataset for both $M=10$ and $M=100$. Our algorithm also exhibits similar results for $\gamma=0.5$ and both values of $M$. However, by selecting $\gamma=0.1$ for $M=10$ and $\gamma=0.05$ for $M=100$, based on the result in Corollary \ref{cor:1}, we achieve the highest learning rates while minimizing communication costs. This result highlights the efficiency and robustness of our algorithm in handling non-IID data distribution by carefully choosing $\gamma$ based on the model size and worker number.

\begin{figure}[!t]
    \centerline{
    \subfloat[$M=10$ and MNIST dataset]{%
        \includegraphics[width=0.5\columnwidth]{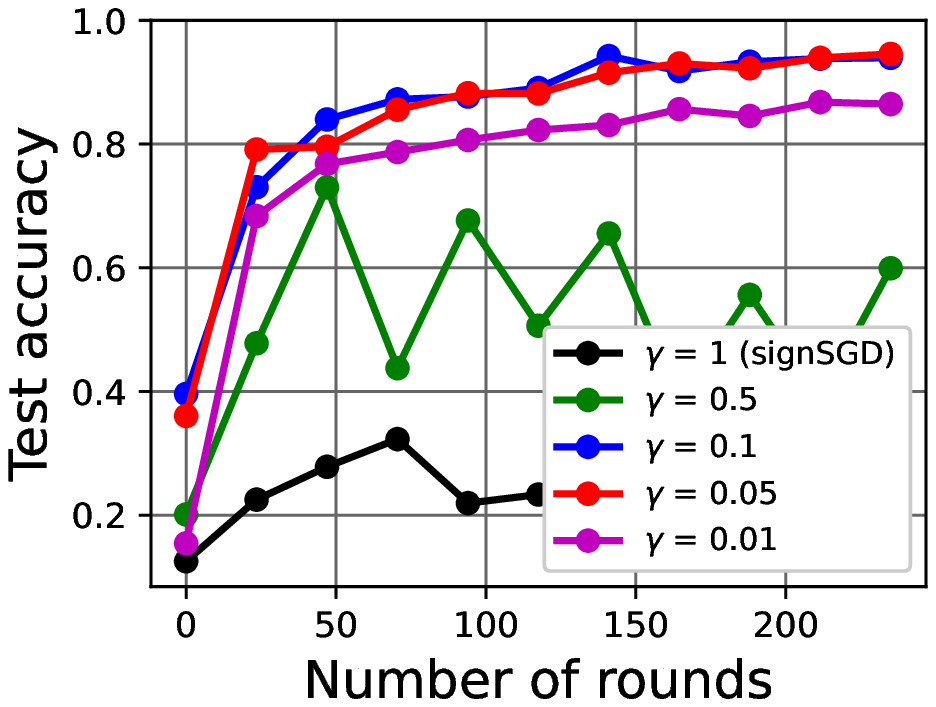}}
    \hfill
    \subfloat[$M=100$ and MNIST dataset]{%
        \includegraphics[width=0.5\columnwidth]{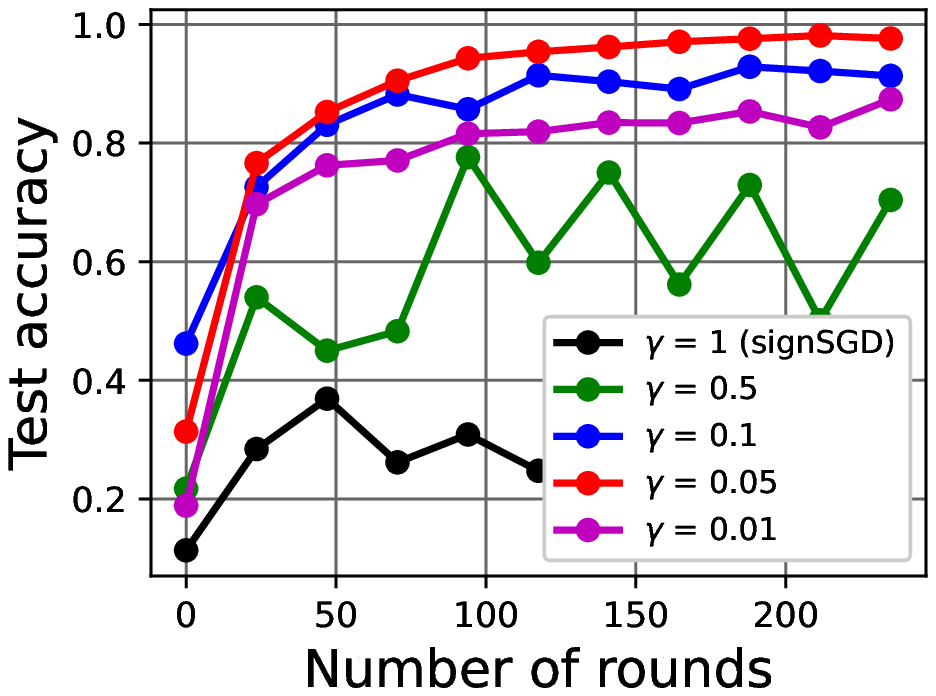}}
    }
   \caption{
    The test accuracy versa training rounds with the different number of workers $M$ for the non-IID data distribution. 
    }\label{fig:s3gd_nonIID} 
    \vspace{-0.3cm}
\end{figure}

\begin{figure}[!t]
    \centerline{
    \subfloat[$M=10$ and MNIST dataset]{%
        \includegraphics[width=0.5\columnwidth]{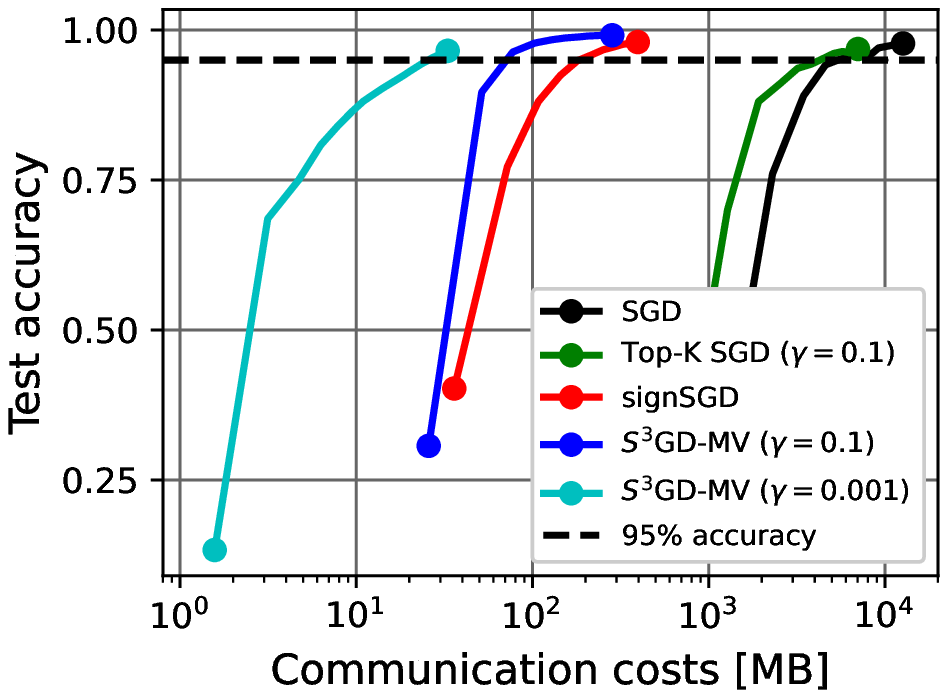}}
    \hfill
    \subfloat[$M=100$ and MNIST dataset]{%
        \includegraphics[width=0.5\columnwidth]{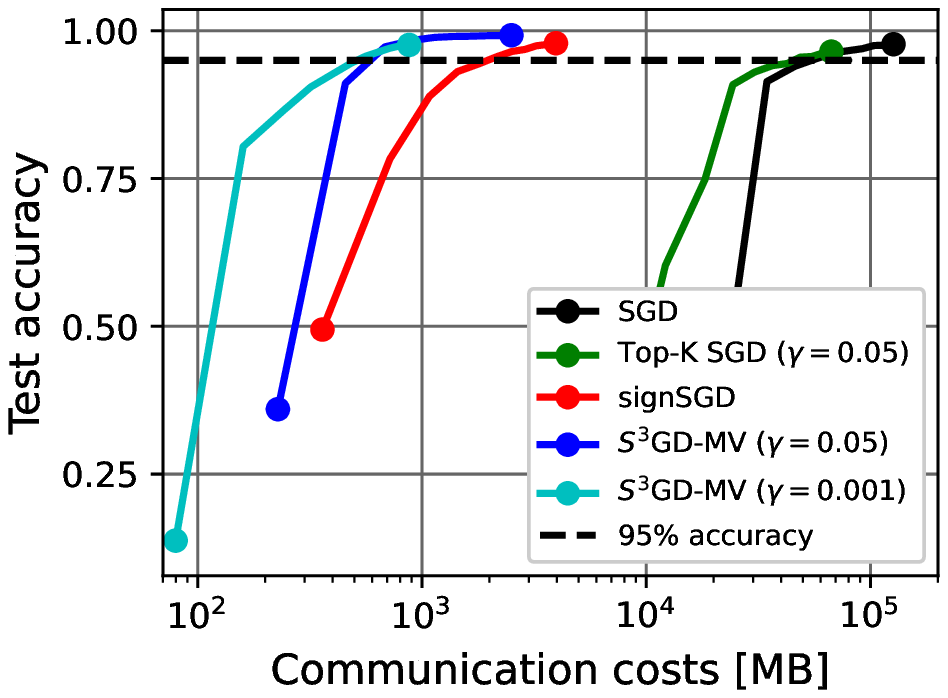}}
    }
    \caption{
    The test accuracy versa the total communication cost. We compare four distributed learning algorithms for the IID data setting.
    }\label{fig:algs} 
    \vspace{-0.3cm}
\end{figure}

\subsection{Communication cost reduction}

 
 It is important to compare the algorithms for the test accuracy versa the total communication cost for distributed learning. The total communication cost is computed by multiplying the number of communication rounds by the cost per iteration. The test accuracy comparison of different algorithms based on the total communication costs is illustrated in Fig. \ref{fig:algs}. We compare our algorithm using the parameters of $(M, \gamma)=(10, 0.001)$, $(M, \gamma)=(10, 0.001)$, $(M, \gamma)=(100, 0.001)$ and $(M, \gamma)=(100, 0.05)$ with benchmark algorithms, including the vanilla SGD, Top-$K$ SGD with memory, and signSGD.  The total communication cost for each algorithm is summarized in Table \ref{tab:alg}. 
 
The results shown in Fig. \ref{fig:algs}-(a) demonstrate the superior performance of $\mathsf{S}^3$GD-MV in terms of both test accuracy and communication cost reduction. When $(M, \gamma)=(10, 0.1)$, the proposed algorithm achieves the highest test accuracy while reducing the total communication cost by roughly 100x, 80x, and 3x compared to the vanilla SGD, top-$K$ SGD, and signSGD, respectively. We can reduce the communication cost further at the expense of the test accuracy slignly by reducing $\gamma$. For instance, when $(M, \gamma)=(10, 0.001)$, to reach a 95$\%$ test accuracy, our algorithm saves the communication cost by about 300x, 200x, and 10x compared to the vanilla SGD, top-$K$ SGD, and signSGD, respectively.  

This remarkable improvement is still attained even when $M=100$ as shown in Fig. \ref{fig:algs}-(b). For example, when $(M, \gamma)=(100, 0.05)$, to attain a 95$\%$ test accuracy, our algorithm requires about 200x, 100x, and 5x less communication costs than the vanilla SGD, top-$K$ SGD, and signSGD algorithms. These results clearly demonstrate the effectiveness of the sparsification and sign quantization techniques employed in $\mathsf{S}^3$GD-MV.



\begin{figure}[!t]
    \centerline{
    \subfloat[$\mathsf{TopKSign}$ and MNIST dataset]{%
        \includegraphics[width=0.5\columnwidth]{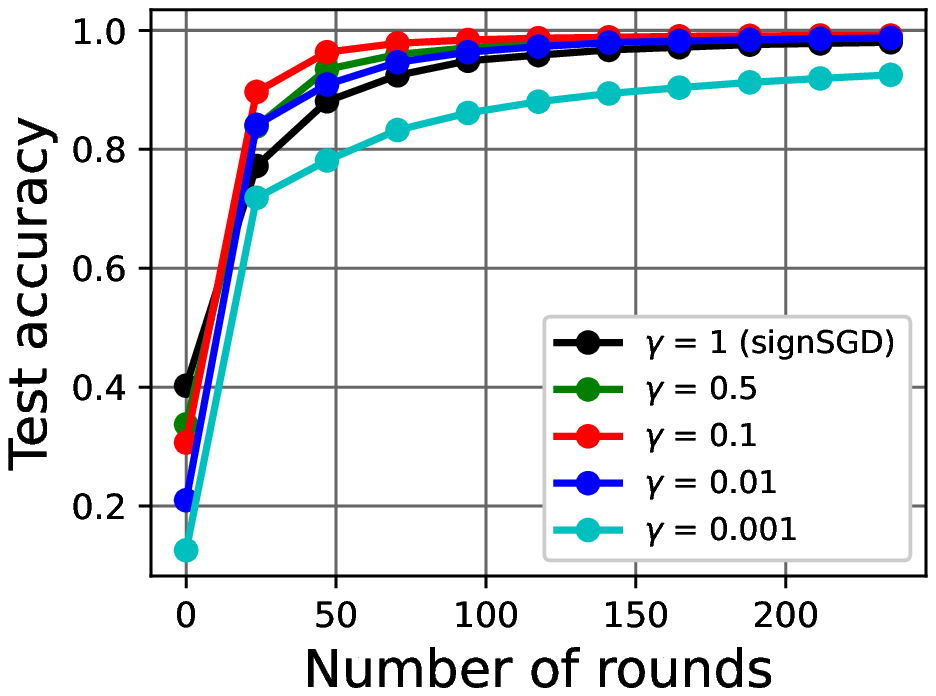}}
    \hfill
    \subfloat[$\mathsf{RandKSign}$ and MNIST dataset]{%
        \includegraphics[width=0.5\columnwidth]{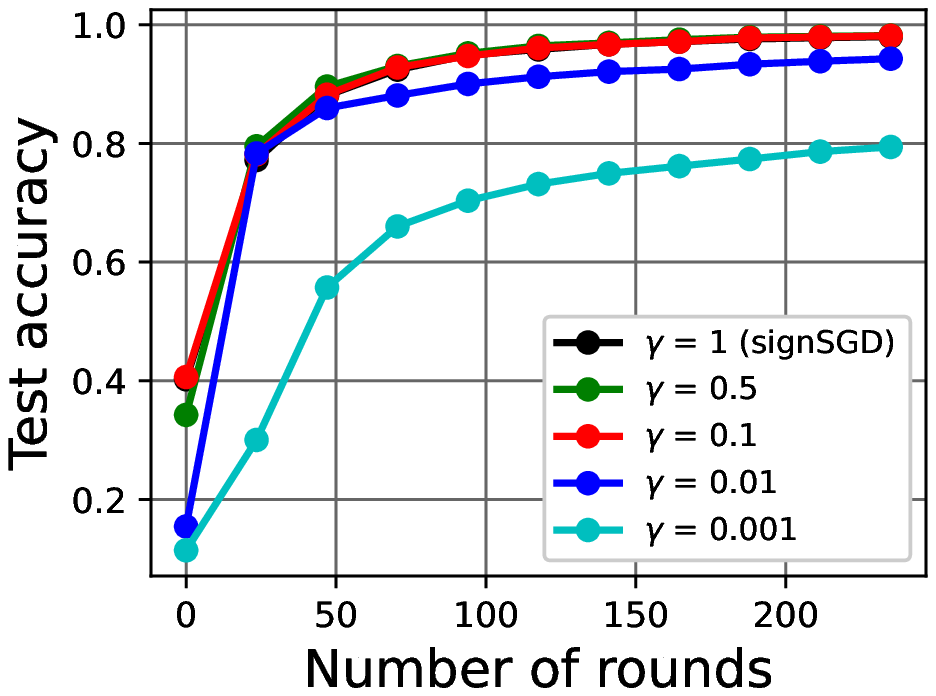}}
    }\vfill
    \centerline{
    \subfloat[$\mathsf{TopKSign}$ and CIFAR10 dataset]{%
        \includegraphics[width=0.5\columnwidth]{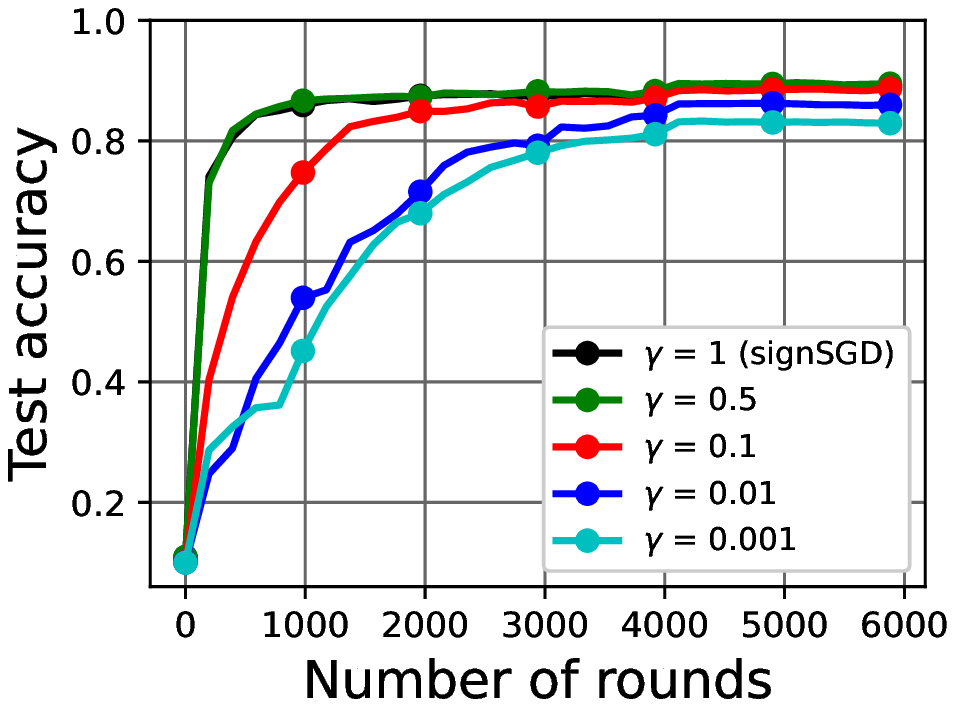}}
    \hfill
    \subfloat[$\mathsf{RandKSign}$ and CIFAR10 dataset]{
        \includegraphics[width=0.5\columnwidth]{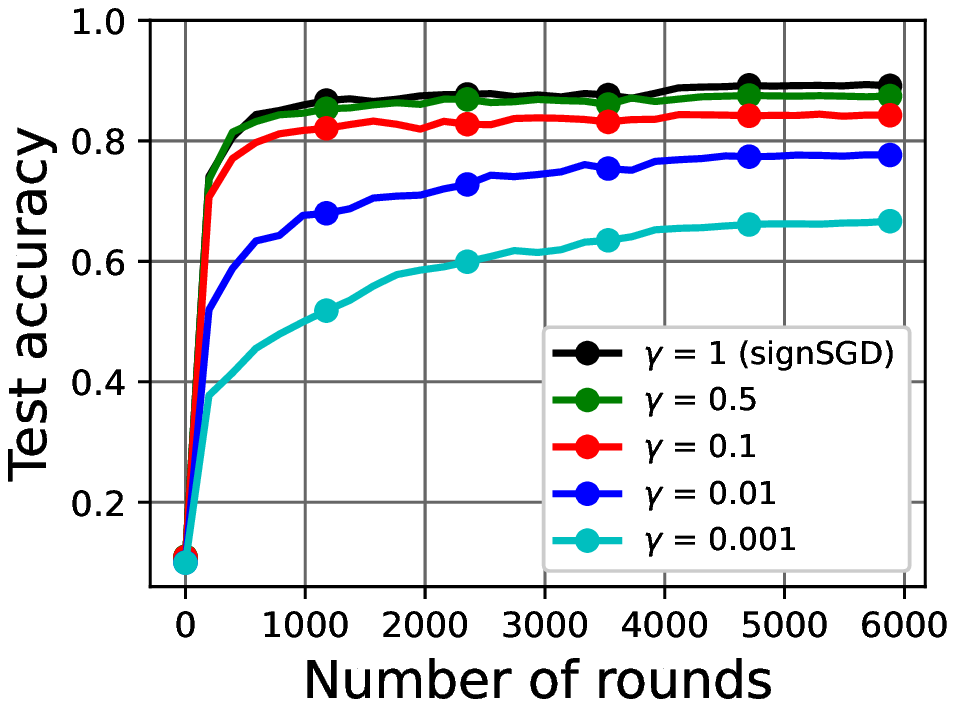}}
    }
    \caption{
    The test accuracy versa training rounds with the different gradient compression schemes for $M=10$ and IID data distribution.
    }\label{fig:top_rand} 
    \vspace{-0.3cm}
\end{figure}

 \subsection{Effects of random selection}
We compare the test accuracy of two algorithms, Algorithm 1 using ${\sf TopKSign}$ and Algorithm 2 using ${\sf RandKSign}$, using both the IID MNIST and CIFAR10 datasets. As shown in the figures, both algorithms have the same communication cost, but their test accuracy performance varies based on the value of $\gamma$. For Algorithm 1 with ${\sf TopKSign}$, there is an optimal value of $\gamma$ that provides the best learning performances, which depends on both the number of workers $M$ and the model size $N$. However, when using Algorithm 2 with ${\sf RandKSign}$, the test accuracy continues to improve as $\gamma$ increases for both datasets. This result indicates that better test accuracy can be achieved by increasing communication cost, which is in contrast to the case with ${\sf TopKSign}$, where a specific value of $\gamma$ provides the highest learning performance while diminishing the communication cost.

 \subsection{Effects of optimization hyper-parameters}
When optimizing the deep learning model, it is also important to investigate the effects of hyper-parameters, including learning rates and momentum term, which may accelerate the convergence rate and improve the accuracy. To accomplish this, we modify the model update rule of the proposed algorithm in \eqref{eqn:s3gd} as \begin{align}
     \mathbf{x}^{t+1} = \mathbf{x}^t - \delta^t \mathbf{v}^t,
 \end{align}
 where $\mathbf{v}^t$ is the momentum added at iteration $t$, which is defined as 
 \begin{align}
     \mathbf{v}^t = \mu \mathbf{v}^{t-1} + \mathsf{sgn} \left[ \sum_{m=1}^M \mathsf{TopKSign} \left( \mathbf{g}_m^t \right) \right].
 \end{align}
The test accuracy results for different hyper-parameters $\eta$ and $\mu$ are plotted in Fig. \ref{fig:hyperparam}. The addition of momentum leads to a faster convergence rate, however, the final accuracy falls short when compared to results without momentum. It is evident from the results that an optimal selection of hyper-parameters can improve the learning performance. Despite the influence of hyper-parameters on the learning performance, our algorithm still outperforms all benchmark algorithms in terms of communication cost reduction. For the purpose of generating the previous figures, the hyper-parameters $\eta = 1$ and $\mu = 0$ were utilized. 
 

\begin{figure}[t]
    \centerline{
    \subfloat[$M=1$ and $\gamma = 0.5$]{%
        \includegraphics[width=0.5\columnwidth]{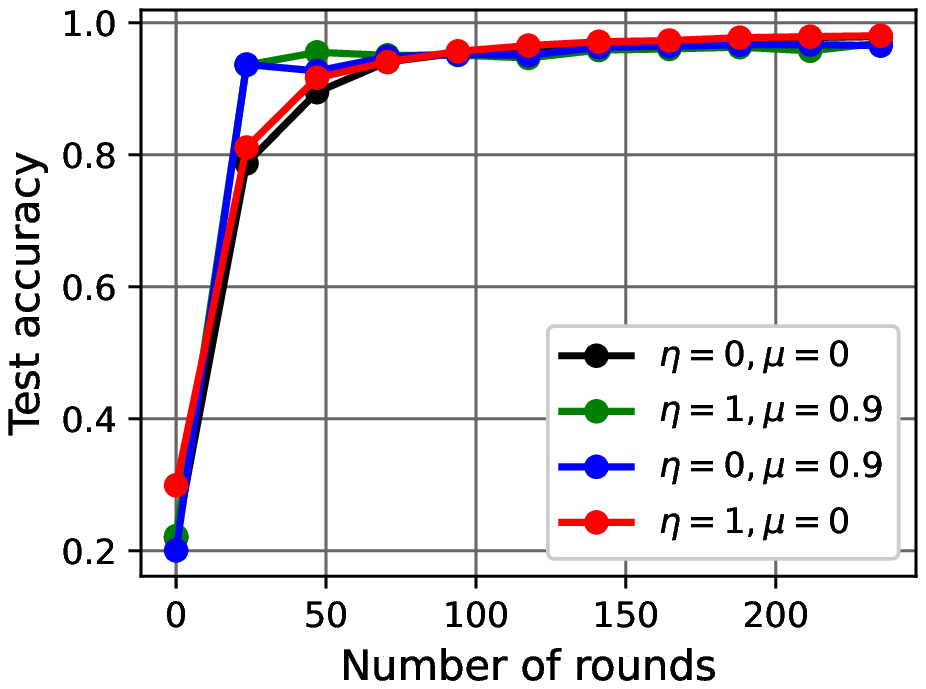}}
    \hfill
    \subfloat[$M=1$ and $\gamma = 0.01$]{%
        \includegraphics[width=0.5\columnwidth]{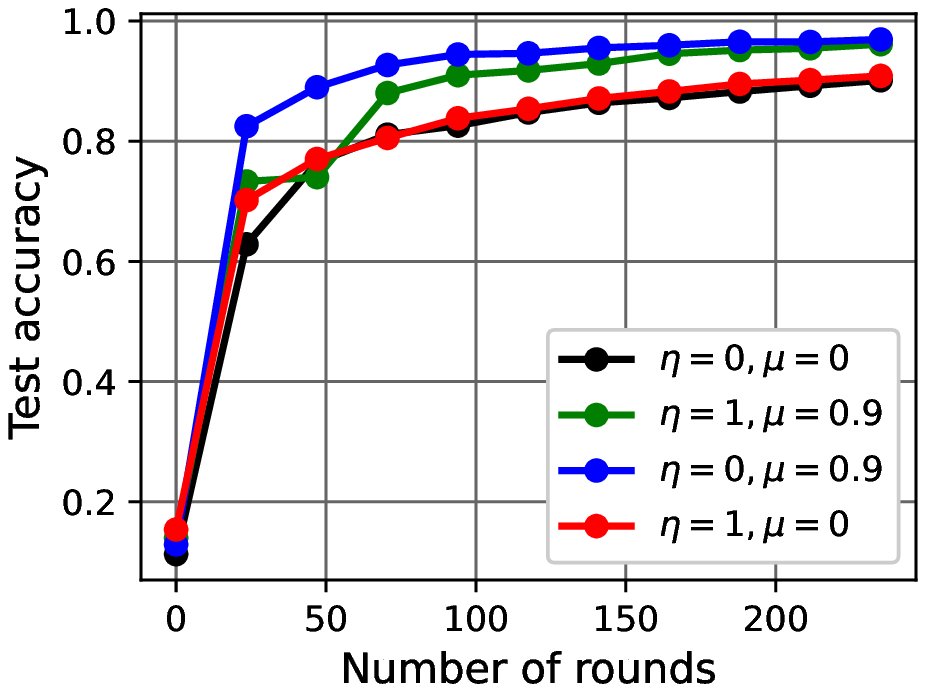}}
    }\vfill
    \centerline{
    \subfloat[$M=10$ and $\gamma = 0.5$]{%
        \includegraphics[width=0.5\columnwidth]{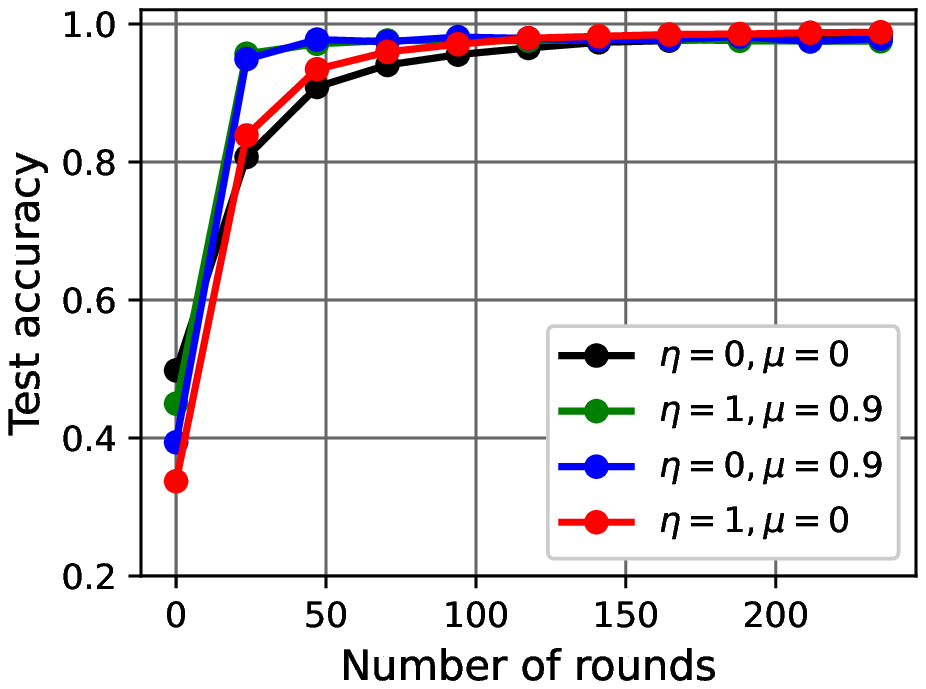}}
    \hfill
    \subfloat[$M=10$ and $\gamma = 0.01$]{
        \includegraphics[width=0.5\columnwidth]{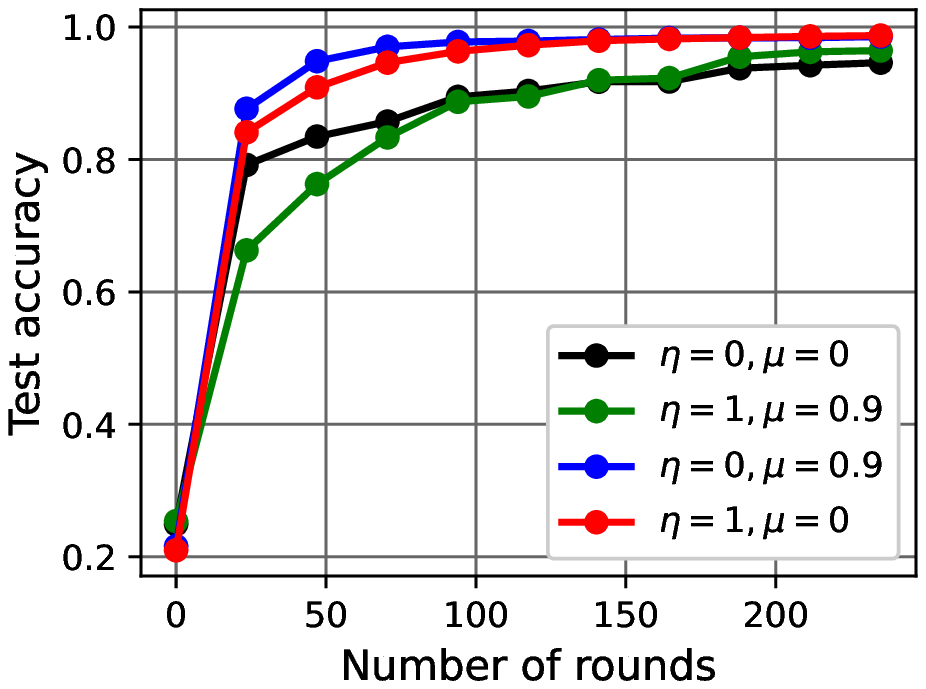}}
    }
    \caption{
    Test accuracy comparison with the different hyper-parameters: the error weight $\eta$ and the momentum $\mu$ for MNIST dataset. 
    }\label{fig:hyperparam}
\end{figure}

\section{Proofs of Theorems and Lemmas} \label{sec:pfs}
This section provides the poofs for our main analytical results stated in Theorems and Lemmas in the previous sections.

\subsection{Proof of Lemma \ref{lem:1}}
\label{pf:lem:1}

\begin{proof}
The second moment of the stochastic gradient $ g_{m,n}^t$ is lower bounded as       \begin{align*}
    \mathbb{E} \left[ \left| g_{m,n}^t \right|^2 \right] &= \int_0^\infty \mathbb{P} \left[ \left| g_{m,n}^t \right|^2 \ge x \right] dx \\
    & \hspace{-0.1em} \stackrel{(a)}{=} \int_0^{\left(\rho_{m,n}^t(\gamma)\right)^2} \! \mathbb{P} \left[ \left| g_{m,n}^t \right|^2 \ge x \right] dx \\
    & \hspace{1em} + \int_{\left(\rho_{m,n}^t(\gamma)\right)^2}^\infty \! \mathbb{P} \left[ \left| g_{m,n}^t \right|^2 \ge x \right] dx \\
    & \ge \int_0^{\left(\rho_{m,n}^t(\gamma)\right)^2} \mathbb{P} \left[ \left| g_{m,n}^t \right|^2 \ge x \right] dx \\
    & \ge \left(\rho_{m,n}^t(\gamma)\right)^2 \, \mathbb{P} \left[ \left| g_{m,n}^t \right| \ge \rho_{m,n}^t(\gamma) \right] \\
    & = \left(\rho_{m,n}^t(\gamma)\right)^2 \, \mathbb{P} \left[ \mathsf{TopKSign} \left( g_{m,n}^t \right) \ne 0 \right] \\
    & \hspace{-0.1em} \stackrel{(b)}{=} \gamma \left(\rho_{m,n}^t(\gamma)\right)^2,  \numberthis{} \label{eq:lb_rho1}     
\end{align*}
where (a) is true because $ \left| g_{m,n}^t \right|^2$ is a positive random variable and (b) follows from Assumption \ref{ass:4}. Since the top-$K$ threshold $\rho_{m,n}^t(\gamma)$ can be chosen to be inversely proportional to the sparsity parameter $\gamma$, we can express a lower bound of $\rho_{m,n}^t(\gamma)$ in terms of a polynomial function of $\gamma$ as 
\begin{align}
    \rho_{m,n}^t(\gamma) \geq \epsilon' \frac{1}{\gamma^{\ell}}, \label{eq:lb_rho2}
\end{align}
 for some constant $\epsilon' \ge 0$ and positive value $\ell \in \mathbb{R}^{+}$. Plugging \eqref{eq:lb_rho2} into \eqref{eq:lb_rho1}, we have
\begin{align}
    \left( \epsilon' \right)^2 \frac{1}{\gamma^{2\ell-1}}  \leq 	\mathbb{E} \left[ \left| g_{m,n}^t \right|^2 \right].
\end{align}   
Furthermore, from Assumption \ref{ass:3}, $\mathbb{E} \left[ \left| g_{m,n}^t \right|^2 \right] \leq \left| \bar{g}_n^t \right|^2 + \sigma_n^2$, we get the upper bound of $\left( \epsilon' \right)^2 \frac{1}{\gamma^{2\ell-1}}$ as
\begin{align}
    \left( \epsilon' \right)^2 \frac{1}{\gamma^{2\ell-1}} \leq \left| \bar{g}_n^t \right|^2 + \sigma_n^2.
\end{align}
By setting $\ell = 1/2$, we have
\begin{align}
    \epsilon' \leq \left| \bar{g}_n^t \right| \sqrt{1 + \frac{\sigma_n^2}{\left| \bar{g}_n^t \right|^2}}.
\end{align}
Since $\sqrt{1+ \sigma_n^2/\left|\bar{g}_n^t\right|^2} \geq 1$ for all $n\in [N]$,  this inequality always holds as long as we choose $\epsilon' $ is smaller than $\left| \bar{g}_n^t \right| $. Consequently, if we set $\epsilon = \epsilon' / \left| \bar{g}_n^t \right|$ and $\ell = 1/2$, $\rho_{m,n}^t(\gamma)$ is lower bounded by $\rho_{m,n}^t(\gamma) \ge \frac{\epsilon}{\sqrt{\gamma}} \left| \bar{g}_n^t \right|$, which completes the proof. 
\end{proof}

\subsection{Proof of Lemma \ref{lem:2}}
\label{pf:lem:2}

\begin{proof}
We are interested in computing the probability that the sign of the true gradient $ \mathsf{sign} \left( \bar{g}_n^t \right)$ is different with the sign of the stochastic gradient chosen by the top-$K$ operator $ \mathsf{TopKSign} \left( g_{m,n}^t \right)$. This sign flipping error occurs if the gap between the stochastic gradient and the true one $ \left| g_{m,n}^t - \bar{g}_n^t \right|$ is larger than the sum of the magnitude of the true one and the threshold, as depicted in Fig. \ref{fig:signflip}. With the threshold used for the top-$K$ sparsification operation in \eqref{eqn:topk}, $\rho_{m,n}^t(\gamma)$, we can upper bound it by 
\begin{align*} \label{eqn:proofA_3}
    & \mathbb{P} \left[ \mathsf{TopKSign} \left( g_{m,n}^t \right) \ne \mathsf{sign} \left( \bar{g}_n^t \right) \right] \\
    & \hspace{6em} \leq \mathbb{P} \left[ \left| g_{m,n}^t - \bar{g}_n^t \right| \ge \left| \bar{g}_n^t \right| + \rho_{m,n}^t(\gamma) \right]. \numberthis{}
\end{align*}

\begin{figure}[t]
  \centering
  \includegraphics[width=1\linewidth]{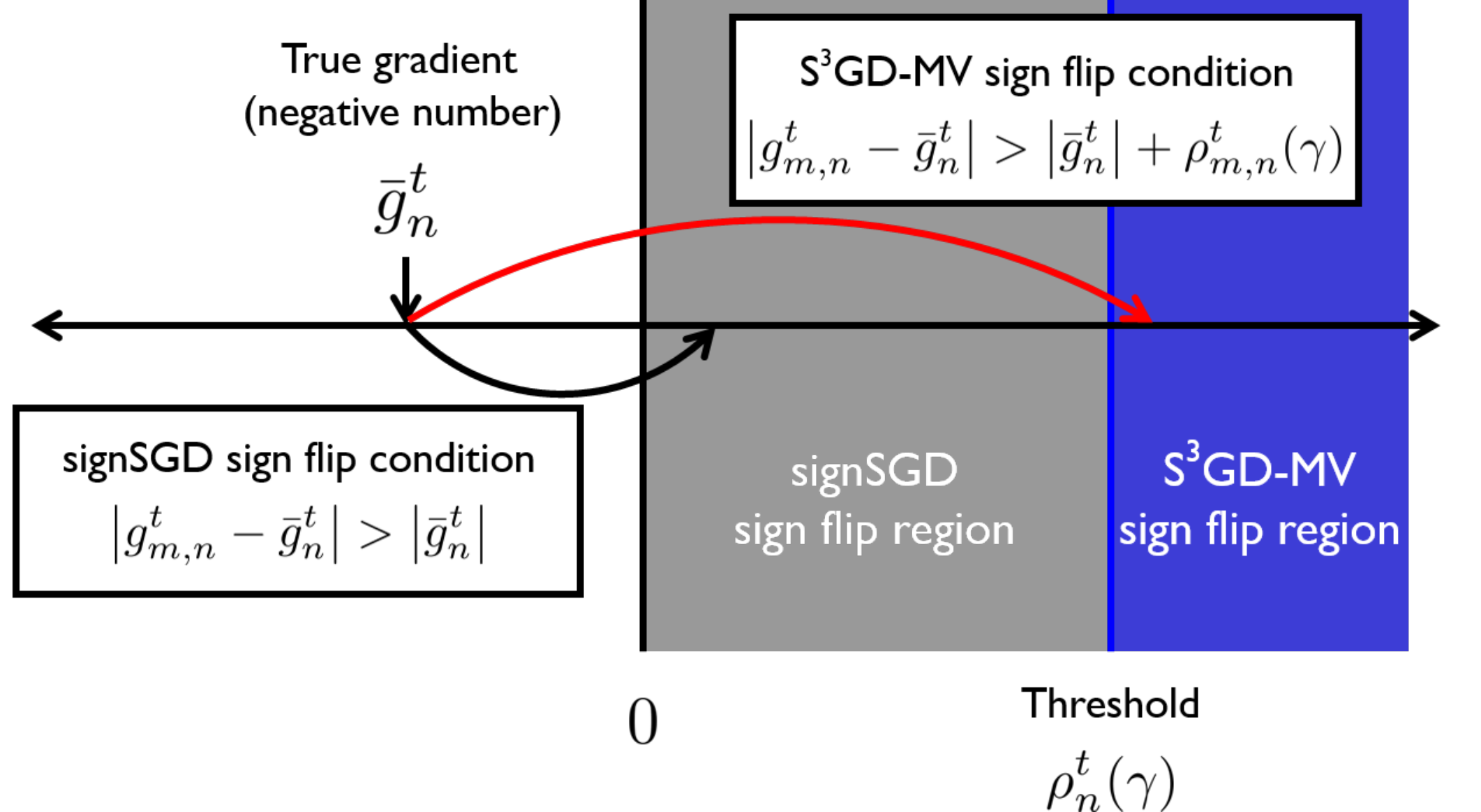}
  \caption{An illustration for the sign flip events of $\mathsf{S}^3$GD-MV and signSGD.}
  \label{fig:signflip}
\end{figure}

\noindent Then, applying the Markov's and Jensen's inequalities, this probability is further upper bounded as           
\begin{align*} 
    \mathbb{P} \left[ \left| g_{m,n}^t - \bar{g}_n^t \right| \ge \left| \bar{g}_n^t \right| + \rho_{m,n}^t(\gamma) \right] &\leq \dfrac{\mathbb{E} \left[ \left| g_{m,n}^t - \bar{g}_n^t \right| \right]}{\left| \bar{g}_n^t \right| + \rho_{m,n}^t(\gamma)} \\
    & \leq \dfrac{\sqrt{\mathbb{E} \left[ \left| g_{m,n}^t - \bar{g}_n^t \right|^2 \right]}}{\left| \bar{g}_n^t \right| + \rho_{m,n}^t(\gamma)} \\
    & \leq \dfrac{\sigma_n^t}{\left| \bar{g}_n^t \right| + \rho_{m,n}^t(\gamma)}, \numberthis{} \label{eq:error} 
\end{align*}
where the last inequality follows from Assumption \ref{ass:3}.  To get the upper bound, we need a lower bound of $\rho_{m,n}^t(\gamma)$, which is stated in Lemma \ref{lem:1}. Invoking \eqref{eq:lem1} into \eqref{eq:error}, we get the upper bound of the sign flipping error probability by the top-$K$ sparsification as
\begin{align*} \label{eq:signerror1}
    \mathbb{P} \left[ \mathsf{TopKSign} \left( g_{m,n}^t \right) \ne \mathsf{sign} \left( \bar{g}_n^t \right) \right] &\leq \dfrac{1}{1+\frac{\epsilon}{\sqrt{\gamma}}} \cdot \frac{\sigma_n^t}{\left| \bar{g}_n^t \right|} \\
    & = \dfrac{1}{\sqrt{B^t}\left( 1+\frac{\epsilon}{\sqrt{\gamma}} \right)} \cdot \frac{\sigma_n}{\left| \bar{g}_n^t \right|}, \numberthis{} 
\end{align*}
where the last equality follows from the sample variance of the mini-batch with size $B^t$, i.e., $\sigma_n^t = \sigma_n / \sqrt{B^t}$. 
\end{proof}

\subsection{Proof of Lemma \ref{lem:4}}
\label{pf:lem:4}

\begin{proof}
Let $\mathcal{Z}_n^t$ be a collection of the workers whose sign of the stochastic gradient is not aligned with that of the true gradient at iteration $t$, i.e., 
\begin{align} \label{eqn:pfC_Z}
    \mathcal{Z}_n^t = \left\{ m\in [M]:  \mathsf{TopKSign} \left( g_{m,n}^t \right)   \neq \mathsf{sign}\left({\bar g}_{n}^t\right)\right\}\subset \mathcal{M}_n^t
\end{align} 
with cardinality $Z_n^t \triangleq |\mathcal{Z}_n|$. Notice that $Z_n^t$ is the sum of $M_n^t$ independent Bernoulli trials with the sign flip error probability $p_{m,n}^t= \mathbb{P} \left[ \mathsf{TopKSign} \left( g_{m,n}^t \right) \neq \mathsf{sign} \left( \bar{g}_n^t \right) \right] $, i.e., $Z_n^t \sim {\sf B}(M_n^t , p_{m,n}^t)$. Then, the sign decoding error by the majority vote can be calculated as
\begin{align} \label{eqn:notalign}
    \mathbb{P} \left[  {\sf sgn} \! \left( \! \sum_{m \in \mathcal{M}_n^t} \!\!\! {\sf TopKSign}(g_{m,n}^t) \right) \! \neq \!  {\sf sign}\left( {\bar g}_n^t \right) \right] & \! = \! \mathbb{P} \left[ Z_n^t \! \geq \! \frac{M_n^t}{2} \right]. 
\end{align}

Applying Chernoff-Hoeffding inequality \cite{hoeffding1963probability}, we get the bound of the error probability as
\begin{align}
    \mathbb{P}\left[Z_n^t \geq \frac{1}{2}M_{n}^t\right] \leq \exp \left(-M_n^t {\sf KL} \left( \frac{1}{2} || p_{m,n}^t\right)\right),
\end{align}
where ${\sf KL}\left(\frac{1}{2}||p_{m,n}^t\right)$ is the Kullback-Leibler distance between two Bernoulli random variables of parameter 1/2 and $p_{m,n}^t$, defined as
\begin{align}
    {\sf KL}\left(\frac{1}{2}||p_{m,n}^t\right) = \frac{1}{2}\ln \frac{1}{2p_{m,n}^t} +\frac{1}{2}\ln \frac{1}{2(1-p_{m,n}^t)}.
\end{align}
Since ${\sf KL}\left(\frac{1}{2}||p_{m,n}^t\right) = \frac{1}{2}\ln \frac{1}{4p_{m,n}^t(1-p_{m,n}^t)}$, this upper bound simplifies to 
\begin{align} 	
    \mathbb{P}\left[Z_n^t \geq \frac{1}{2}M_{n}^t\right]& \leq  \left[4(1-p_{m,n}^t)p_{m,n}^t \right]^{\frac{M_n^t}{2}}.
\end{align}
This completes the proof. 
\end{proof}

\subsection{Proof of Theorem \ref{thm:1}}
\label{pf:thm:1}

\begin{proof}
Recall that from Lemma \ref{lem:3},  $M_n^t$ denotes the number of workers that select the $n$-th gradient component by the top-$K$ sparsification at iteration $t$, and it is distributed as
\begin{align}
    \mathbb{P} \left[ M_n^t = u \right] = \binom{M}{u} \gamma^u(1-\gamma)^{M-u}.
\end{align} 
for $u \in \{0,1,\ldots, M\}$. Then, from \eqref{eqn:s3gd}, \eqref{eqn:assumption2} and \eqref{eqn:notalign}, conditioned on ${\bf x}^t$, the expectation of $f^{t+1} - f^t$ is upper bounded by
    \begin{align}
        &\mathbb{E} \left[ \left. f^{t+1} - f^t \right| {\bf x}^t \right] \nonumber \\
        &\leq -\delta^t \lVert \bar{\mathbf{g}}^t \rVert_1 + \frac{\left( \delta^t \right)^2}{2} \lVert \mathbf{L} \rVert_1  \nonumber \\
        & \hspace{1em} + \delta^t \sum_{n=1}^N \left| \bar{g}_n^t \right| \, \mathbb{P} \left[ M_n^t = 0 \right] - \frac{\left( \delta^t \right)^2}{2} \sum_{n=1}^N L_n \, \mathbb{P} \left[ M_n^t = 0 \right]  \nonumber \\
        & \hspace{1em} + 2\delta^t \sum_{n=1}^N \left| \bar{g}_n^t \right| \sum_{u=1}^M \mathbb{P} \left[ \left. Z_n^t \geq \frac{M_n^t}{2} \right| M_n^t = u \right] \mathbb{P} \left[ M_n^t = u \right].
        \label{eqn:proofB_1}
    \end{align}
Applying Lemma \ref{lem:4}, the probability of the sign decoding error conditioned $M_n^t=u$ is 
\begin{align} \label{eqn:pfB_Z2}
    \mathbb{P}\left[ \left. Z_n^t \geq \frac{1}{2}M_{n}^t  \right| M_n^t =u \right] &\leq  \left[4(1-p_{m,n}^t)p_{m,n}^t \right]^{\frac{u}{2}} \nonumber\\
    &\leq  \left[4(1-{\tilde p}_{m,n}^t){\tilde p}_{m,n}^t \right]^{\frac{u}{2}}, 
\end{align}
where the last inequality follows from Lemma \ref{lem:2}, i.e., $p_{m,n}^t \leq {\tilde p}_{m,n}^t= \frac{\sigma_n}{\sqrt{B^t} \left( 1 + \frac{\epsilon}{\sqrt{\gamma}} \right) \left| \bar{g}_n^t \right|} <\frac{1}{2}$. To obtain a compact expression, we derive a more loose upper bound on the decoding error probability as     
\begin{align} \label{eqn:pfD_1}
    \mathbb{P}\left[ \left. Z_n^t \geq \frac{1}{2}M_{n}^t  \right| M_n^t = u \right]& \leq  \left[4(1-{\tilde p}_{m,n}^t){\tilde p}_{m,n}^t \right]^{\frac{u}{2}}. \nonumber\\
    & \hspace{-0.1em} \stackrel{(a)}{\leq}   \left[4 {\tilde p}_{m,n}^t \right]^{\frac{u}{2}} \nonumber\\
    & \hspace{-0.1em} \stackrel{(b)}{\leq} \frac{{\tilde p}_{m,n}^t}{\sqrt{u}},
\end{align}
where (a) is valid when ${\tilde p}_{m,n}^t<1/4$ and (b) is because $ \left[4 {\tilde p}_{m,n}^t \right]^{\frac{M_n^t}{2}} \leq \frac{{\tilde p}_{m,n}^t}{\sqrt{M_n^t}}$  for $0<{\tilde p}_{m,n}^t< \left(2^{M_n^t} \sqrt{M_n^t}\right)^{-\frac{2}{M_n^t-2}}$ with the limit value of $\lim_{M_n^t\rightarrow \infty}\left(2^{M_n^t} \sqrt{M_n^t}\right)^{-\frac{2}{M_n^t-2}}=1/4$. From Lemma \ref{lem:2} and \ref{lem:3}, this sign decoding probability is upper bounded as
\begin{align}
    &\sum_{u=1}^M \mathbb{P} \left[ \left. Z_n^t \geq \frac{1}{2} M_n^t \right| M_n^t = u \right] \mathbb{P} \left[ M_n^t = u \right] \nonumber \\
    & \hspace{8em} \leq \sum_{u=1}^M \frac{{\tilde p}_{m,n}^t}{\sqrt{u}}  \binom{M}{u} \gamma^u (1-\gamma)^{M-u} \nonumber\\
    & \hspace{8em} = \frac{\sigma_n}{\sqrt{B^t} \left( 1 + \frac{\epsilon}{\sqrt{\gamma}} \right) \left| \bar{g}_n^t \right|} \beta(M,\gamma), \label{eq:decode_error}
\end{align}
    where the last equality follows by defining $\beta(M,\gamma)$ as
\begin{align}
    \beta(M,\gamma) = \sum_{i=1}^M \frac{1}{\sqrt{u}} \binom{M}{u} \gamma^u (1-\gamma)^{M-u}.
\end{align}
Invoking \eqref{eq:decode_error} into \eqref{eqn:proofB_1}, we can rewrite the upper bound of $\mathbb{E} \left[ \left. f^{t+1} - f^t \right| {\bf x}^t \right]$ in a compact form:
\begin{align} \label{eqn:proofB_2} 
    \mathbb{E} \left[ \left. f^{t+1} - f^t \right| {\bf x}^t \right] &\leq \alpha(M,\gamma) \, \delta^t \lVert \bar{\mathbf{g}}^t \rVert_1 + \alpha(M,\gamma) \, \frac{\left( \delta^t \right)^2}{2} \lVert \mathbf{L} \rVert_1 \nonumber \\
    & \hspace{1em} + 2\delta^t  \frac{\lVert \boldsymbol{\sigma} \rVert_1}{\sqrt{B^t} \left( 1 + \frac{\epsilon}{\sqrt{\gamma}} \right)  } \beta(M,\gamma),
\end{align}
where
\begin{align}
    \alpha(M,\gamma) = 1 - (1-\gamma)^M.
\end{align}
Plugging the step size $\delta^t = \frac{1}{\sqrt{T \lVert \mathbf{L} \rVert_1}}$ and the batch size $B^t = T$ into \eqref{eqn:proofB_2}, we obtain 
\begin{align}
    \mathbb{E} \left[ \left. f^{t+1} - f^t \right| {\bf x}^t \right] &\leq - \frac{\alpha(M,\gamma)}{\sqrt{T \lVert \mathbf{L} \rVert_1}} \lVert \bar{\mathbf{g}}^t \rVert_1 + \frac{\alpha(M,\gamma)}{2T} \nonumber \\
    & \hspace{1.2em} + \frac{\beta(M,\gamma)}{T\sqrt{ \lVert \mathbf{L} \rVert_1}} \cdot \frac{2}{ 1 + \frac{\epsilon}{\sqrt{\gamma}} } \lVert \boldsymbol{\sigma} \rVert_1. 
\end{align}
Applying the method of telescoping sums over $t\in [T]$, we get the lower bound  of $f^0 - f^\star$ as    
\begin{align} \label{eq:th2_final}
    & f^0 - f^\star \geq f^0 - \mathbb{E} \left[ f^T \right] \nonumber \\
    & \hspace{3.5em} = \mathbb{E} \left[ \sum_{t=0}^{T-1} f^t - f^{t+1} \right]  \nonumber\\
    & \hspace{3.5em} \geq \mathbb{E} \left[ \sum_{t=0}^{T-1} \left\{ \frac{\alpha(M,\gamma)}{\sqrt{T \lVert \mathbf{L} \rVert_1}} \lVert \bar{\mathbf{g}}^t \rVert_1 - \frac{\alpha(M,\gamma)}{2T} \right. \right. \nonumber \\ 
    & \hspace{9em} \left. \left. - \frac{\beta(M,\gamma)}{T\sqrt{ \lVert \mathbf{L} \rVert_1}}  \frac{2}{  \left( 1 + \frac{\epsilon}{\sqrt{\gamma}} \right)} \lVert \boldsymbol{\sigma} \rVert_1 \right\} \right]  \nonumber \\
    & \hspace{3.5em} = \alpha(M,\gamma) \left\{ \sqrt{\frac{T}{\lVert \mathbf{L} \rVert_1}} \mathbb{E} \left[ \frac{1}{T} \sum_{t=0}^{T-1} \lVert \bar{\mathbf{g}}^t \rVert_1 \right] - \frac{1}{2} \right.  \nonumber \\
    & \hspace{8em} \left. - \frac{\beta(M, \gamma)}{ \alpha(M,\gamma) \sqrt{ \lVert \mathbf{L} \rVert_1}} \frac{2}{1 + \frac{\epsilon}{\sqrt{\gamma}}} \lVert \boldsymbol{\sigma} \rVert_1\right\}. 
\end{align}
By re-arranging \eqref{eq:th2_final}, we conclude that 

\begin{align}  
    & \mathbb{E} \left[ \frac{1}{T} \! \sum_{t=0}^{T-1} \lVert \bar{\mathbf{g}}^t \rVert_1 \right]  \nonumber\\
    &\leq \! \frac{1}{\sqrt{T}} \! \left[ \! \sqrt{\lVert \mathbf{L} \rVert_1} \! \left( \! \frac{f^0 \! - \! f^\star }{\alpha(M,\gamma)}  + \! \frac{1}{2} \right) \! + \! \frac{\beta(M,\gamma)}{  \alpha(M,\gamma)}   \frac{2}{ 1\!+\!\frac{\epsilon}{\sqrt{\gamma}}} \lVert \boldsymbol{\sigma} \rVert_1 \right].
\end{align}
 This completes the proof. 
 \end{proof}

 \subsection{Proof of Corollary \ref{cor:1}}
 \label{pf:cor:1}

\begin{proof}
 When $\gamma \ll 1$, the convergence upper bound in \eqref{eq:thm1} can be represented as
\begin{align} \label{eqn:upperbound}
    \mathbb{E} \left[ \frac{1}{T} \! \sum_{t=0}^{T-1} \lVert \bar{\mathbf{g}}^t \rVert_1 \right] &\leq \frac{1}{\sqrt{T}} \! \left( \frac{f^0 - f^\star}{M\gamma} \sqrt{\lVert \mathbf{L} \rVert_1} + \frac{2\sqrt{\gamma}}{\epsilon} \lVert \boldsymbol{\sigma} \rVert_1 \right) \nonumber \\
    & \triangleq h(\gamma).
\end{align}
From \eqref{eqn:upperbound}, by taking the derivative with respect to $\gamma$ we obtain as
\begin{align}
    \frac{\partial}{\partial \gamma} h(\gamma) = \frac{1}{\sqrt{T}} \left( -\frac{f^0 - f^\star}{M \gamma^2} \sqrt{\lVert \mathbf{L} \rVert_1} + \frac{1}{\epsilon \sqrt{\gamma}} \lVert \boldsymbol{\sigma} \rVert_1 \right).
\end{align}
Then, setting it to zero $\left. \frac{\partial}{\partial \gamma} h(\gamma) \right|_{\gamma = \gamma^\star} = 0$, we attain the stationary point $\gamma^\star$ as
\begin{align}
    \gamma^\star = \left( \frac{\epsilon \left( f^0 - f^\star \right)}{M} \frac{\sqrt{\lVert \mathbf{L} \rVert_1}}{\lVert \boldsymbol{\sigma} \rVert_1} \right)^{\frac{2}{3}}.
\end{align}
 To verify the convexity of $h(\gamma)$, we compute the second-order derivative of $h(\gamma)$ as
\begin{align} \label{eqn:second}
    \frac{\partial^2}{\partial \gamma^2} h(\gamma) = \frac{1}{\sqrt{T}} \left( \frac{2 \left( f^0 - f^\star \right)}{M \gamma^3} \sqrt{\lVert \mathbf{L} \rVert_1} - \frac{1}{2\epsilon \gamma^{\frac{3}{2}}} \lVert \boldsymbol{\sigma} \rVert_1 \right),
\end{align}
and if we substitute $\gamma^\star$ into \eqref{eqn:second}, we obtain
\begin{align}
    \left. \frac{\partial^2}{\partial \gamma^2} h(\gamma) \right|_{\gamma = \gamma^\star} = \frac{3M}{2\epsilon^2 \left( f^0 - f^\star \right)} \frac{\lVert \boldsymbol{\sigma} \rVert_1^2}{\sqrt{\lVert \mathbf{L} \rVert_1}} > 0,
\end{align}
which means $\gamma = \gamma^\star$ is locally convex. Moreover, we can easily check that $\frac{\partial}{\partial \gamma} h(\gamma) > 0$ if $\gamma > \gamma^\star$, and $\lim_{\gamma \rightarrow \infty} \frac{\partial}{\partial \gamma} h(\gamma) = 0$.
Consequently, it is verified that $\gamma^\star$ is an optimal sparsity of $h(\gamma)$ when $\gamma \ll 1$, which concludes the proof. 
\end{proof}

\subsection{Proof of Theorem \ref{thm:2}}
\label{pf:thm:2}

\begin{proof}
By following the update rule in Algorithm \ref{alg:s3gd_rand}, we can modify \eqref{eqn:proofB_1} as
\begin{align*} \label{eqn:proofE_1}
        &\mathbb{E} \left[ \left. f^{t+1} - f^t \right| {\bf x}^t \right] \\
        &\leq -\delta^t \lVert \bar{\mathbf{g}}^t \rVert_1 + \frac{\left( \delta^t \right)^2}{2} \lVert \mathbf{L} \rVert_1 \\
        & \hspace{1em} + \delta^t \sum_{n=1}^N \left| \bar{g}_n^t \right| \, \mathbb{P} \left[ M_n^t = 0 \right] - \frac{\left( \delta^t \right)^2}{2} \sum_{n=1}^N L_n \, \mathbb{P} \left[ M_n^t = 0 \right] \\
        & \hspace{1em} + 2\delta^t \sum_{n=1}^N \left| \bar{g}_n^t \right| \sum_{u=1}^M \mathbb{P} \left[ \left. \tilde{Z}_n^t \geq \frac{M_n^t}{2} \right| M_n^t = u \right] \mathbb{P} \left[ M_n^t = u \right], \numberthis{}
    \end{align*}
where $\tilde{Z}_n^t$ is the $\mathsf{RandKSign} (\cdot)$ version of $Z_n^t$ in \eqref{eqn:pfC_Z}. 
Since the $\mathsf{RandK}$ operator uniformly select the gradient components, we obtain 
\begin{align} \label{eqn:proofE_2}
    \mathbb{P} \left[ M_n^t = 0 \right] = (1-\gamma)^M.
\end{align}
In addition, we can ignore the the magnitude threshold $\rho_{m,n}^t(\gamma)$ by setting it to zero $\rho_{m,n}^t(\gamma)=0$. From \eqref{eqn:proofA_3}, we can express the probability of the sign decoding error for the  single worker case as
\begin{align} \label{eqn:proofE_3}
        \mathbb{P} \left[ \mathsf{RandKSign} \left( g_n^t \right) \ne \mathsf{sign} \left( \bar{g}_n^t \right) \right] \leq \dfrac{\sigma_n}{\sqrt{B^t} \left| \bar{g}_n^t \right|}. \numberthis{}
    \end{align}
By plugging \eqref{eqn:proofE_3} into \eqref{eqn:pfD_1}, the sign error probability is given by 
\begin{align} \label{eqn:proofE_4}
    \mathbb{P} \left[ \tilde{Z}_n^t \geq \frac{M_n^t}{2} \right] \leq \frac{\sigma_n}{\sqrt{M_n^t B_n^t} \left| \bar{g}_n^t \right|}.
\end{align}
We want to highlight that this sign flip error probability is identical to the result in \cite{bernstein2018asignsgd}. By substituting \eqref{eqn:proofE_2} and \eqref{eqn:proofE_4}, we can rearrange \eqref{eqn:proofE_1} as
    \begin{align*} \label{eqn:proofE_5}
        \mathbb{E} \left[ \left. f^{t+1} - f^t \right| {\bf x}^t \right] & \leq \alpha(M,\gamma) \, \delta^t \lVert \bar{\mathbf{g}}^t \rVert_1 + \alpha(M,\gamma) \, \frac{\left( \delta^t \right)^2}{2} \lVert \mathbf{L} \rVert_1  \\
        & \hspace{1em} + \beta(M,\gamma) \frac{2\delta^t}{\sqrt{B^t}} \lVert \boldsymbol{\sigma} \rVert_1. \numberthis{}
    \end{align*}
With the same learning rate and batch size used to derive in Theorem \ref{thm:1}, \eqref{eqn:proofE_5} can be rewritten as
\begin{align*} \label{eqn:proofE_6}
    \mathbb{E} \left[ \left. f^{t+1} - f^t \right| {\bf x}^t \right] &\leq - \frac{\alpha(M,\gamma)}{\sqrt{T \lVert \mathbf{L} \rVert_1}} \lVert \bar{\mathbf{g}}^t \rVert_1 + \frac{\alpha(M,\gamma)}{2T} \\
    & \hspace{1em} + \frac{2 \beta(M,\gamma)}{T\sqrt{\lVert \mathbf{L} \rVert_1}} \lVert \boldsymbol{\sigma} \rVert_1. \numberthis{}
\end{align*}
Applying the telescoping sum over the entire iteration, we finally obtain 
\begin{align*} \label{eqn:proofE_7}
    f^0 - f^\star &\geq f^0 - \mathbb{E} \left[ f^T \right] \\
    & = \mathbb{E} \left[ \sum_{t=0}^{t-1} f^t - f^{t+1} \right] \\
    & \geq \mathbb{E} \left[ \sum_{t=0}^{T-1} \left\{ \frac{\alpha(M,\gamma)}{\sqrt{T \lVert \mathbf{L} \rVert_1}} \lVert \bar{\mathbf{g}}^t \rVert_1 - \frac{\alpha(M,\gamma)}{2T} \right. \right. \\
    & \hspace{10em} \left. \left. - \frac{2 \beta(M,\gamma)}{T\sqrt{\lVert \mathbf{L} \rVert_1}} \lVert \boldsymbol{\sigma} \rVert_1 \right\} \right] \\
    & = \alpha(M,\gamma) \left\{ \sqrt{\frac{T}{\lVert \mathbf{L} \rVert_1}} \mathbb{E} \left[ \frac{1}{T} \sum_{t=0}^{T-1} \lVert \bar{\mathbf{g}}^t \rVert_1 \right] - \frac{1}{2} \right. \\
    & \hspace{8em} \left. - \frac{2\beta(M,\gamma)}{\alpha(M, \gamma) \sqrt{\lVert \mathbf{L} \rVert_1}} \lVert \boldsymbol{\sigma} \rVert_1 \right\}. \numberthis{}
\end{align*}
Rearranging \eqref{eqn:proofE_7}, we obtain the expression in Theorem \ref{thm:2}, which completes the proof. 
\end{proof}

\section{Conclusion}
\label{sec:conclusion}

We have introduced a communication-efficient distributed learning algorithm called ${\sf S}^3$GD-MV. ${\sf S}^3$GD-MV provides two synergistic benefits from sparsification and sign quantization under the majority vote principle: First, it diminishes the communication cost considerably than the state-of-the-art algorithms. Second, it achieves a convergence rate faster than the existing methods with a proper selection of the sparsification parameter $K$ in terms of model size $N$ and the number of workers $M$. We have theoretically and empirically demonstrated these synergistic gains through the convergence rate analysis and simulations.

One promising future work is to investigate the synergistic gain of sparsification and sign quantization in federated learning problems and to show the convergence rate in this more general setting, which is expected to broaden possible applications of ${\sf S}^3$GD-MV. It is also interesting to investigate the robustness of the algorithm for adversarial attacks.

\ifCLASSOPTIONcaptionsoff
  \newpage
\fi



%



\bibliographystyle{IEEEtran}
\bibliography{bibfile}

\end{document}